\documentclass[a4paper,11pt,english]{article}

\usepackage[a4paper,lmargin=2.9cm,rmargin=2.9cm]{geometry}
\usepackage{algorithmic, algorithm}
\usepackage{xspace}
\usepackage{amsmath, amssymb, amsthm, dsfont, mathtools}
\usepackage{url}
\usepackage{enumitem}

\allowdisplaybreaks[3]

\newtheorem{myremark}{Remark}

\newcommand{\expect}[1]{E(#1)}
\newtheorem{theorem}{Theorem}
\newtheorem{lemma}[theorem]{Lemma}

\newtheorem{corollary}[theorem]{Corollary}
\newtheorem{remark}[theorem]{Remark}

\title{General Drift Analysis with Tail Bounds\footnote{A preliminary version 
of this paper appeared in the proceedings of ISAAC~2014 \cite{LehreWittISAAC2014}.}}
			
\author{Per Kristian Lehre\\
School of Computer Science\\
University of Birmingham\\
Birmingham, B15 2TT\\
United Kingdom\\
P.K.Lehre@birmingham.ac.uk
	\and
Carsten Witt\\
DTU Compute\\
Technical University of Denmark\\
2800 Kgs.\ Lyngby\\
Denmark\\
cawi@dtu.dk}

\newcommand{\ea}{(1+1)~EA\xspace}

\newcommand{\wrt}{w.\,r.\,t.\xspace}

\newcommand{\ie}{i.\,e.\xspace}
\newcommand{\eg}{e.\,g.\xspace}
\newcommand{\LO}{\textsc{Leading\-Ones}\xspace}

\newcommand{\ONEMAX}{\textsc{OneMax}\xspace}
\newcommand{\OneMax}{\ONEMAX}

\newcommand{\N}{\mathds{N}}
\newcommand{\R}{\mathds{R}}
\newcommand{\filt}{\mathcal{F}_t}

\newcommand{\filtzero}{\mathcal{F}_0}
\newcommand{\filtcond}[1]{\,;\,{#1}\mid\filt}

\newcommand{\filtuc}[1]{\mathcal{F}_{#1}}

\newcommand{\indic}[1]{\mathds{1}\left\{#1\right\}}
\newcommand{\semc}{\;;\;}

\DeclareMathOperator{\Prob}{Pr}

\newcommand{\E}[1]{\mathord{E}\mathord{\left(#1\right)}}
\newcommand{\bigO}[1]{\mathord{O}\mathord{\left(#1\right)}}
\newcommand{\littleo}[1]{\mathord{o}\mathord{\left(#1\right)}}

\newenvironment{proofof}[1]{\begin{proof}[#1]}{\end{proof}}

\newcommand{\expec}[1]{\mathord{E}\mathord{\left(#1\right)}}

\newenvironment{statement}{\begin{trivlist}\item\itshape}{\end{trivlist}}

\newcommand{\xmin}{x_{\min}}
\newcommand{\xmax}{x_{\max}}

\begin{document}
\maketitle

\begin{abstract}
Drift analysis is one of the state-of-the-art techniques for the
runtime analysis of randomized search heuristics (RSHs) such as
evolutionary algorithms (EAs), simulated annealing etc.  The vast majority
of existing drift theorems yield bounds on the expected value of the
hitting time for a target state, \eg, the set of optimal solutions,
without making additional statements on the distribution of this
time. We address this lack by providing a general drift theorem that
includes bounds on the upper and lower tail of the hitting time
distribution. The new tail bounds are applied to prove very precise
sharp-concentration results on the running time of a simple
EA on standard benchmark problems, including
the class of general linear functions.
Surprisingly, the probability
of deviating by an $r$-factor in lower order terms of the expected
time decreases exponentially with $r$ on all these
problems.
The usefulness of the theorem outside the theory of RSHs is
demonstrated by deriving tail bounds on the number of cycles in random permutations.
All these results
handle a position-dependent (variable) drift that was not covered by 
previous drift theorems with tail bounds.
Moreover, our theorem can be specialized into virtually all existing drift theorems with drift towards the
target from the literature. Finally, user-friendly specializations
of the general drift theorem are given.
\end{abstract}

\section{Introduction}

Randomized search heuristics (RSHs) such as simulated annealing,
evolutionary algorithms (EAs), ant colony optimization etc. 
are highly popular techniques in
black-box optimization, \ie, the problem of optimizing a function with only
oracle access to the function. These heuristics often imitate some
natural process, and are rarely designed with analysis in mind. Their
extensive use of randomness, such as in the mutation operator, render
the underlying stochastic processes non-trivial.
While the theory of RSHs is less developed than the theory of
classical, randomized algorithms, significant progress has been made
in the last decade
\cite{AugerDoerrBook,NeumannWittBook,JansenBook}. This theory has
mainly focused on the optimization time, which is
the random variable $T_{A,f}$ defined as the
number of oracle accesses the heuristic $A$ makes before the maximal
argument of $f$ is found. 
Most studies considered the expectation of $T_{A,f}$,
however more information about the distribution of the optimisation
time is  often needed. For example, the expectation can be deceiving
when the runtime distribution has a high variance. Also, tail bounds can be helpful for
other performance measures, such as fixed-budget computation
which seeks to estimate the approximation-quality as a function of time
\cite{DJWZGECCO13}.

 Results on the runtime of RSHs were obtained after relevant analytical techniques were
 developed, some adopted from other fields, others developed
 specifically for RSHs. 
Drift analysis, which is a central method
for analyzing the hitting time of stochastic processes, was introduced  
to the analysis of simulated annealing as early as in~1988 \cite{SasakiHajek1988}. 
Informally, it allows long-term properties of a discrete-time stochastic process
$(X_t)_{t\in\N_0}$ to be inferred from properties of the one-step
change $\Delta_{t} := X_t-X_{t+1}$. In the context of EAs, one has been particularly interested in the random
variable $T_a$ defined as the smallest $t$ such that $X_t\leq a$. 
For example, if $X_t$ represents the ``distance'' of the current
solution in iteration~$t$ to an optimum, then $T_0$ is the
optimization time.

Since its introduction to evolutionary computation by He and Yao in
2001 \cite{HeYao01}, drift analysis has been widely used to analyze the
optimization time of EAs. Many drift
theorems have been introduced, such as
additive drift theorems \cite{HeYao01}, multiplicative drift \cite{DoerrGoldbergAdaptive,DJWMultiplicativeAlgorithmica}, 
variable drift \cite{Johannsen10,MitavskiyVariable,RoweSudholtChoice}, and
population drift \cite{Lehre2010PopNegDrift}. Different assumptions 
and notation used in these theorems 
make it hard to abstract out a unifying statement. 

 Drift analysis is also used outside theory of RSHs,
 for example in queuing theory \cite{Coffman1999,Eryilmaz2012}. 
 The widespread use of these techniques
 in separated research fields has made it 
 difficult to get an overview of the drift theorems. 
 Drift analysis 
 is also related to other areas, such
 as stochastic differential equations and stochastic difference
 relations.

Most drift theorems used in the theory of RSHs relate to the
expectation of the hitting time $T_a$, and there are fewer results
about the tails $\Prob(T_a>t)$ and $\Prob(T_a<t)$. From the simple
observation that
$\Prob(T_a>t) \leq \Prob(\sum_{i=0}^t \Delta_i < a-X_0)$, 
the problem is reduced to bounding the deviation of a sum of random
variables. If the $\Delta_t$ were independent and identically
distributed, then one would be in the familiar scenario of
Chernoff/Hoeffding-like bounds. The stochastic processes originating
from RSHs are rarely so simple, in particular the $\Delta_t$ are often
dependent variables, and their distributions are not explicitly
given. However, bounds on the form
$\expec{\Delta_t\mid X_t}\geq h(X_t)$ for some function $h$ often hold. The drift
is called \emph{variable} when $h$ is a non-constant function. The
variable drift theorem provides bounds on the expectation of
$T_a$ given some conditions on $h$. However, there have been no
general tail bounds from a variable drift condition. The only results
in this direction seem to be the tail bounds for
probabilistic recurrence relations from \cite{KarpJACM94}; however,
this scenario corresponds to the specific case of
random variables $X_t$ that are monotonically decreasing over time.

Our main contribution is a new, general drift theorem that provides sharp
concentration results for the hitting time of stochastic processes
with variable drift, along with concrete advice and examples how to apply it.
The theorem is used to bound the tails of the
optimization time of the well-known \ea \cite{Droste2002} to the 
benchmark problems \OneMax and \LO, as well as the class of linear functions, which 
is an intensively studied problem in the area \cite{WittCPC13}. 
Surprisingly, the results show that the distribution is highly
concentrated around the expectation. The probability of deviating by
an $r$-factor in lower order terms decreases exponentially with $r$. 
In a different application outside the theory of RSHs, we use drift analysis to analyze probabilistic 
recurrence relations and show that the number of cycles in a random permutation
of $n$ elements is sharply concentrated around the expectation $\ln n$.
As a secondary 
contribution, we prove that our general drift theorem can 
be specialized into virtually all variants of drift theorems with drift towards the target 
(in particular, variable, additive, and multiplicative drift) 
that have been scattered over the literature on runtime analysis of 
RSHs. Unnecessary assumptions such as discrete or finite search spaces will 
be removed from these theorems.

This paper is structured as follows. 
Section~\ref{sec:prel} introduces notation and basics 
of drift analysis. Section~\ref{sec:tail-variable} presents the general drift theorem with tail bounds and 
suggestions for user-friendly corollaries. Section~\ref{sec:applyingtail} 
applies the tail bounds from our theorem. Sharp-concentration results on the running time of 
the \ea on \OneMax, \LO and general 
linear functions are obtained. The application outside the theory of RSHs 
with respect to random recurrence relations is 
described at the end of this section (Section~\ref{sec:recurrences}). 
 In all these applications, the probability of deviating by
an $r$-factor in lower order terms of the expected time 
decreases exponentially with $r$. 
 \ref{sec:specialcasesvariable} demonstrates the generality 
 of the theorem by identifying a large number of drift theorems from the literature as special cases.
We finish with some conclusions.

\section{Preliminaries}
\label{sec:prel}

We analyze time-discrete stochastic processes represented 
by a sequence of non-negative random variables $(X_t)_{t\in\N_0}$. 
For example, $X_t$ could represent 
a certain distance value of an RSH from an optimum. 
In particular, $X_t$ might aggregate several different random 
variables realized by an RSH at time~$t$ into a single one. In contrast to 
existing drift theorems, we do not demand that the state space is discrete (\eg, all non-negative integers) 
and do not require either that the state space is subset of a compact interval. Instead, the state space is 
a subset of the real numbers, possibly bounded on one side.

We adopt  the convention that the process should pass below some threshold $a\ge 0$ (``minimizes'' 
its state)  
and define the first hitting time $T_a:=\min\{t\mid X_t\le a\}$. 
If the actual 
process seeks to maximize its state, typically a straightforward mapping 
allows us to stick to the convention of minimization. 
In an important special case, we are interested in the hitting time~$T_0$ of target state~$0$; for 
example when a \ea, a very simple RSH, is run on the well-known \OneMax problem 
and were are interested in the first 
point of time where the number of zero-bits becomes zero. 
Note that 
$T_a$ is a stopping time and that we assume that 
the stochastic process is adapted to some filtration $(\filt)_{t\in\N_0}$,
such as its natural filtration $\sigma(X_0,\dots,X_t)$.

Our main goal is to describe properties of the distribution of the
first hitting time $T_a$, hence some information about the stochastic
process before that time is required. In particular, we consider the
expected one-step change of the process
\[\delta_t := \E{X_t - X_{t+1} \semc X_t>a \mid \filt},\] the so-called 
\emph{drift}.  For any event $A$ and random variable $X$, we use the
well-established notation $\E{X\semc A\mid \filt}:=\E{X\indic{A}\mid \filt}$, where
$\indic{}$ is the
indicator function.  Note that $\delta_t$ in general is a random
variable since the outcomes of $X_0,\dots,X_t$ are random. 
Suppose we
manage to bound the random variable $\delta_t$ from below by some real number $\delta^*>0$, conditioning on that 
$X_t\ge a$. This is the same as bounding 
\[
\E{X_t - X_{t+1} -\delta^* \semc X_t>a \mid \filt} \ge 0,
\]
except for the case that $\Prob(X_t>a)$, where the conditioning does not work; however, this difference 
is unimportant for our analysis of first hitting time. 
 Then, informally speaking, we know that the
process, conditioned on not having reached the target, 
decreases its state (``progresses towards~$0$'') in
expectation by at least~$\delta^*$ in every step, and the additive
drift theorem (see Theorem~\ref{theo:additive} below) will provide a
bound on $T_0$ that only depends on~$X_0$ and~$\delta^*$. In fact, the
very natural-looking result $\E{T_0\mid \filtzero} \le X_0/\delta^*$ will be
obtained. However, bounds on the drift might be more complicated.  For
example, a bound on $\delta_t$ might depend on~$X_t$ or states at even
earlier points of time, \eg, if the progress decreases as the current
state decreases. This is often the case in applications to
EAs. However, for such algorithms the whole ``history'' is rarely
needed. Simple EAs and other RSHs are Markov processes such that often
$\delta_t=\E{X_t-X_{t+1} \semc X_t>a \mid X_t}$ for an appropriate $X_t$.

With respect to Markov processes on 
discrete search spaces, drift conditions traditionally 
use conditional expectations 
	such as $\E{X_t - X_{t+1} \mid X_t=i}$ and bound these 
	for  all $i>a$ where $\Prob(X_t=i)>0$, leading to statements like 
$\E{X_t - X_{t+1} \mid X_t=i} \ge \delta$ instead of 
 instead of directly bounding the random variable
$\E{X_t - X_{t+1} - \delta; X_t>a \mid \filt}$. Note that $\Prob(X_t=i)$ may 
be zero everywhere in continuous search spaces.

As pointed out, the drift $\delta_t$ is in general a random variable and 
should not be confused with the ``expected drift'' $\E{\delta_t}=\E{E(X_t - X_{t+1}; X_t>a \mid \filt)}$, which rarely is available since  
it averages over the whole history of the stochastic process. Drift is based on the inspection of 
the progress from one step to another, taking into account every possible history. This one-step 
inspection often makes it easy to come up with bounds on $\delta_t$. Drift theorems could also 
be formulated based on expected drift, possibly allowing stronger statement on the first hitting time. 
However, in many applications it is infeasible to bound the expected value of the drift in a precise enough way for stronger 
statement to be obtained. See \cite{Jagerskupper11} 
for one of the rare analyses of ``expected drift'', which we will not get into in this paper.

 We now present the first drift theorem for additive 
drift. It is based on 
 \cite{HeYao01}, from which we removed the unnecessary assumptions that
 the search space is discrete and the Markov property. We only demand
 a bounded state space for the lower bound. The proof is partially inspired by similar 
 formulations of the additive drift theorem presented and discussed  in \cite{LenglerStegerDriftArxiv} and 
 \cite{LenglerDriftArxiv}.

 \begin{theorem}[Additive Drift, following \cite{HeYao01}]
 \label{theo:additive}
 Let $(X_t)_{t\in\N_0}$, be a stochastic process, adapted to a filtration $(\filt)_{t\in\N_0}$, over some state space
 $S\subseteq \R$, and let $b,\delta_{\mathrm{u}},\delta_{\mathrm{\ell}}>0$.  
 Then for $T_0:=\min\{t\mid X_t\le 0\}$ it holds:
 \begin{enumerate}[leftmargin=!,labelwidth=7mm,label=(\roman*)]
 \item  \label{itm:D1}
 If  $E(X_{t}-X_{t+1} - \delta_{\mathrm{u}} \semc X_t > 0 \mid\filt)   \geq 0$ and $X_t\ge 0$ for 
all for all $t\in\N_0$ then 
 $\E{T_0\mid \filtzero} \le \frac{X_0}{\delta_{\mathrm{u}}}$.
 \item \label{itm:D1lower}
 If $E(X_{t}-X_{t+1}-\delta_{\mathrm{\ell}} \semc  X_t > 0\mid\filt )
 \le 0$ and $X_t\leq b$ for all $t\in\N_0$, then 
 $\E{T_0\mid \filtzero} \ge \frac{X_0}{\delta_{\ell}}$.
 \end{enumerate}
 \end{theorem}

 \begin{proof}
   We start by  proving the first statement.
As it does not change the distribution of~$T_0$, we 
consider without loss of generality the stopped process $Y_t\coloneqq X_{t\wedge T_0}$ that does not change 
state after time~$T_0$. Therefore,  for all $t\in\N$, 
\[
\expect{Y_{t+1}-Y_t \;; X_t>0\mid \filtuc{t}} = \expect{Y_{t+1}-Y_t \;; t<T_0\mid \filtuc{t}} = \expect{Y_{t+1}-Y_t\mid \filtuc{t}}.
\]
Hence, Condition \ref{itm:D1} is equivalent to
\begin{equation*}
\expect{Y_{t+1}-Y_t \mid \filtuc{t}}  \le - \delta_{\mathrm{u}} \expect{\indic{t<T_0}\mid\filtuc{t}} 
\end{equation*}
which by rearrangement yields 
\begin{equation}
\expect{Y_{t+1} \mid \filtuc{t}}
\le Y_t - \delta_{\mathrm{u}} \expect{\indic{t<T_0}\mid\filtuc{t}}
\label{eq:add-drift-eq-one}
\end{equation}
Applying the tower property and using \eqref{eq:add-drift-eq-one} again,
\begin{align*}
  \expect{\expect{Y_{t+1}\mid \filtuc{t}}\mid\filtuc{t-1}}
  & \le \expect{ Y_t\mid \filtuc{t-1} } - \delta_{\mathrm{u}} \expect{\indic{t<T_0}\mid\filtuc{t-1}}\\ 
& \le 
Y_{t-1} -\delta_{\mathrm{u}} \expect{\indic{t-1<T_0}\mid\filtuc{t-1}}- \delta_{\mathrm{u}} \expect{\indic{t<T_0}\mid\filtuc{t-1}}
\end{align*}
if $t\ge 1$. Inductively, for all $t\ge 0$,
\[
  \expect{Y_{t+1}\mid \filtuc{0}} \le Y_0 -
  \delta_{\mathrm{u}} \sum_{i=0}^{t} \expect{\indic{i<T_0}\mid\filtuc{0}}
  = Y_0 -   \delta_{\mathrm{u}} \sum_{i=1}^{t} \expect{\indic{i\leq T_0}\mid\filtuc{0}}
\]
Now, since $\expect{Y_{t+1}\mid \filtuc{0}}\ge 0$ for all $t\ge 0$ 
by the assumption $Y_{t+1}\in\R_0^+$,  we obtain 
\[
  0 \leq Y_0 - \delta_{\mathrm{u}} \sum_{i=1}^{t} \expect{\indic{i\leq T_0}\mid\filtuc{0}}
 =  Y_0 -  \delta_{\mathrm{u}} \sum_{i=1}^{t+1}  \Pr(T_{a}\ge i \mid  \filtuc{0})
\]
for all $t\ge 0$. Letting $t\to\infty$, we have
\[
0 \le Y_0 - \delta_{\mathrm{u}} \sum_{i=1}^{\infty}  \Pr(T_{a}\ge i \mid \filtuc{0}) = Y_0 - \delta_{\mathrm{u}} \expect{T_0 \mid \filtuc{0}}.
\]
using the identity $\expect{X}=\sum_{i=1}^\infty \Pr(X\ge i)$ that applies to any random variable~$X$ taking only values 
from the non-negative integers. Also, 
$\expect{T_0\mid \filtuc{0}}<\infty$ since $Y_0$ only takes real values and $\delta_{\mathrm{u}}>0$ is assumed. 

Rearranging terms and substituting $X_0=Y_0$ concludes the proof of
the first statement.

For the second statement, we first use \ref{itm:D1lower} and proceed with reversed inequalities but otherwise 
analogously to the proof of the first statement. Therefore, we obtain for all $t\ge 0$ that
\[
\expect{Y_{t}\mid \filtuc{0}} \ge Y_0 -   \delta_{\mathrm{\ell}} \sum_{i=1}^{t}  \Pr(T_{a}\ge i\mid \filtuc{0}).
\]
Since probabilities are non-negative, this implies 
\[
\expect{Y_{t}\mid \filtuc{0}} \ge Y_0 -   \delta_{\mathrm{\ell}} \sum_{i=1}^{\infty}  \Pr(T_{a}\ge i\mid \filtuc{0}) = X_0 - \delta_{\mathrm{\ell}}\expect{T_0\mid \filtuc{0}}
\]
for all $t\ge 0$ and therefore also
\[
\expect{Y_{t} \mid \filtuc{0}} \ge  Y_0 - \delta_{\mathrm{\ell}}\expect{T_0 \mid \filtuc{0}}
\]
Letting $t\to\infty$, 
\[
\liminf_{t\to\infty}\expect{Y_{t} \mid \filtuc{0}} \ge  Y_0 - \delta_{\mathrm{\ell}}\expect{T_0\mid \filtuc{0}}.
\]

Substituting again $X_t=Y_{t\wedge T_0}$,
it now suffices to prove $\liminf_{t\to\infty}\expect{X_{t\wedge T_0 } \mid \filtuc{0}} \le 0$. 
To this end, we consider the sequence $X_{t\wedge T_0}$, $t\ge 0$, and recall that $X_t\le b<\infty$ for all $t\ge 0$. 
Hence, Conditions~1 and~2 of the dominated convergence theorem (Theorem~\ref{thm:dot} in the appendix) 
have been established with $V\coloneqq b$. Moreover, $\lim_{t\to\infty} X_{t\wedge T_0} = X_{T_0}$, where 
by 
definition $X_{T_0}\le 0$. Now the dominated convergence theorem yields
\begin{align*}
  \liminf_{t\to\infty} \expect{X_{t\wedge T_0}\mid \filtuc{0}}
  & \leq \lim_{t\to\infty} \expect{X_{t\wedge T_0}\mid \filtuc{0}} \\
  &= \expect{\lim_{t\to\infty} X_{t\wedge T_0}\mid \filtuc{0}} \\
  & = \expect{X_{T_0}\mid \filtuc{0}} \le 0.
\end{align*}
This completes the proof of the second statement.
\end{proof}

Summing up, additive drift is concerned with the very simple scenario
that there is a progress of at least~$\delta_{\mathrm{u}}$ from all
non-optimal states towards the target in~$(i)$ and a progress of at
most~$\delta_{\mathrm{\ell}}$ in $(ii)$. Since the $\delta$-values are
independent of~$X_t$, one has to use the worst-case drift
over all non-optimal~$X_t$. This might lead to very bad bounds on the first
hitting time, which is why more general theorems (as mentioned in the
introduction) were developed. Interestingly, these
more general theorems are often proved based on
Theorem~\ref{theo:additive} using an appropriate mapping (sometimes called \emph{Lyapunov function}, \emph{potential function}, \emph{distance 
function}   
or \emph{drift function}) from
the original state space to a new one.  Informally, the mapping
``smoothes out'' position-dependent drift into an (almost)
position-independent drift. 
We will use the same approach in the following.

Before proceeding to the general theorem, we point out the necessity
of a bounded state space for the lower bound. Without assuming
$X_t\leq b$, we cannot in general 
conclude that $\expect{T_0 \mid \filtuc{0}} \geq X_0 / \varepsilon$. 
Let us consider the following Markov chain as an example: 
the state space is $\N_0$ and $a\coloneqq 0$. 
From state $i \in \N$ the chain transits to state~$0$ with probability $1/2$
and to state $2i$  with probability $1/2$. In particular, at time~$t$ the 
process is either at state~$0$ or at state at least $t$. For the drift, we obtain 
$\expect{X_{t+1}-X_t \mid \filtuc{t}} = -X_t/2 + X_t/2 = 0$,   
so clearly the process satisfies $\expect{X_{t+1}-X_{t} + \varepsilon \mid \filtuc{t}} \ge 0$ for any
$\varepsilon>0$. Hence, if the bound 
$\expect{T_0 \mid \filtuc{0}}\geq X_0 / \varepsilon$ was true, we could obtain arbitrarily big lower bounds 
on  $\expect{T_0 \mid \filtuc{0}}$. 
However, the expected first hitting time 
for state $0$, starting from state $1$, is~$2$ 
since every step has a probability 
of~$1/2$ of leading to~$0$.

\section{General Drift Theorem}
\label{sec:tail-variable}

In this section, we present our general drift theorem. As pointed out
in the introduction, we strive for a very general statement, which is
partly at the expense of simplicity. More user-friendly
specializations will be given later. Nevertheless, the underlying idea
of the complicated-looking general theorem is the same as in all drift
theorems. We look into the one-step drift $\delta_t = \E{X_t-X_{t+1} \mid \filt}$, which 
is a random variable that may depend on the complete history of the process up to time~$t$. 
Then we  assume we have a (upper or lower) bound $h(X_t)$ on the drift, formally 
$\delta_t \ge h(X_t)$ or $\delta_t \le h(X_t)$, where the bound depends on~$X_t$ only, 
\ie, a possibly smaller $\sigma$-algebra than~$\filt$.  
Based on~$h$, 
we define a new function $g$ (see Remark~\ref{remark:function-h}),
with the aim of ``smoothing out'' the
dependency, and the drift \wrt\ $g$ (formally,
$\E{g(X_t)-g(X_{t+1})\mid \filt}$) is analyzed. Statements~$(i)$ and
$(ii)$ of the following Theorem~\ref{theo:main} provide bounds
on~$\E{T_0}$ based on the drift \wrt~$g$. In fact, $g$ can be  defined in a
very similar way as in existing variable-drift theorems
\cite{Johannsen10,MitavskiyVariable,RoweSudholtChoice}, such that
Statements~$(i)$ and $(ii)$ can be understood as generalized variable
drift theorems for upper and lower bounds on the expected hitting
time, respectively.

Statements~$(iii)$ and $(iv)$ are concerned with tail bounds on the
hitting time. Here moment-generating functions (mgfs.) of the drift \wrt~$g$
come into play, formally 
\[E(e^{-\lambda (g(X_t)-g(X_{t+1}))}\mid
\filt)\] is bounded. Again for generality, bounds on the
mgf.\ may depend on the point of time~$t$, as
captured by the bounds $\beta_{\mathrm{u}}(t)$ and
$\beta_{\mathrm{\ell}}(t)$. We will see an example in
Section~\ref{sec:applyingtail} where the mapping $g$ smoothes out the
position-dependent drift into a (nearly) position-independent and
time-independent drift, while the mgf.\ of the
drift \wrt~$g$ still heavily depends on the current point of time~$t$
(and indirectly on the position expected at this time).

Our drift theorem generalizes virtually all existing drift theorems concerned 
with a drift towards the target,
including the variable drift theorems for upper
\cite{Johannsen10,RoweSudholtChoice,MitavskiyVariable} and lower
bounds \cite{DFWVariable} (see
Theorem~\ref{theo:variable-rowe-sudholt} and
Theorem~\ref{theo:variable-dfw}), a non-monotone variable drift
theorem \cite{FeldmannKoetzingFOGA13} (see
Theorem~\ref{theo:variablenonmonotone}), and multiplicative drift
theorems \cite{DoerrGoldbergAdaptive,WittCPC13,DoerrDoerrKoetzingGECCO16} (see
Theorem~\ref{theo:multiplicative-drift} and
Theorem~\ref{theo:multdrift-lower}). Our theorem also generalizes
fitness-level theorems \cite{WegenerICALP01,SudholtTEC13} (see
Theorem~\ref{theo:fitnesslevels} and
Theorem~\ref{theo:fitnesslevells-lower}), another well-known technique
in the analysis of randomized search heuristics.
These generalizations are shown in
 \ref{sec:specialcasesvariable}. Note that 
we do not consider the case of 
negative drift (drift away from the target) as studied in \cite{OlivetoW11,OlivetoWittErratumDriftArxiv} 
since this scenario is handled with structurally different techniques.


\begin{myremark}\label{remark:function-h}
  If for some function $h\colon \R_{\ge \xmin}\to\R^+$ where
  $\xmin>0$ and $1/h(x)$ is integrable on $\R_{\ge \xmin}$, either
$    E(X_t-X_{t+1} - h(X_t)  \filtcond{X_t\ge \xmin}) \ge 0$ or 
$    E(X_t-X_{t+1} - h(X_t) \filtcond{X_t\ge \xmin}) \le 0$
  hold, it is recommended to define the function $g$
  in Theorem~\ref{theo:main} as
  \begin{align*}
    g(x) := \frac{\xmin}{h(\xmin)} + \int_{\xmin}^x \frac{1}{h(y)}
    \,\mathrm{d}y   
  \end{align*}
  for $x\geq \xmin$ and $g(0):=0$.
\end{myremark}

\begin{theorem}[General Drift Theorem]
\label{theo:main}
Let $(X_t)_{t\in\N_0}$, be a stochastic process, adapted to a filtration $(\filt)_{t\in\N_0}$, over some state space $S\subseteq \R$.
For some $a\ge 0$, let $T_a=\min\{t\mid X_t\le a\}$. 
Moreover, 
let 
 $g\colon S \to \R_{\ge 0}$ be a function
such that $g(0)=0$ and $g(x) > g(a)$ for all $x > a$.

Then:
\begin{enumerate}[leftmargin=!,labelwidth=7mm]
\item[(i)] 
If   
$E(g(X_t)-g(X_{t+1}) - \alpha_{\mathrm{u}} \filtcond{ X_t >  0})\ge 0$ 
for all $t\in\N_0$ and some $\alpha_{\mathrm{u}}>0$ then 
$E(T_0\mid \filtzero) \le \frac{g(X_0)}{\alpha_{\mathrm{u}}}$.

\item[(ii)] 
If there is $\xmax>0$ such that $g(X_t)\le \xmax$  and 
$E(g(X_t)-g(X_{t+1}) - \alpha_{\mathrm{\ell}} \filtcond{ X_t > 0})\le 0$ 
 for all $t\in\N_0$ and some $\alpha_{\mathrm{\ell}}>0$ then 
$E(T_0\mid \filtzero) \ge \frac{g(X_0)}{\alpha_{\mathrm{\ell}}}$.

\item[(iii)]
If 
there exists $\lambda>0$ and a function $\beta_{\mathrm{u}}\colon \N_0\to\R^+$ 
such that \[E(e^{-\lambda (g(X_t)-g(X_{t+1}))} - \beta_{\mathrm{u}}(t) \filtcond{ X_t >a}) \le 0\] for all $t\in\N_0$, 
 then 
$\Prob(T_a>t^* \mid \filtzero) < \left(\prod_{r=0}^{t^*-1} \beta_{\mathrm{u}}(r)\right) \cdot  e^{\lambda (g(X_0)-g(a))}$
for $t^* > 0$.

\item[(iv)]
If 
there exists $\lambda>0$ and a function $\beta_{\mathrm{\ell}}\colon \N_0\to\R^+$ 
such that \[E(e^{\lambda (g(X_t)-g(X_{t+1}))} - \beta_{\mathrm{\ell}}(t) \filtcond{ X_t > a}) \le 0\] 
for all $t\in\N_0$ then,
$\Prob(T_a < t^* \mid \filtzero ) \le \left( \sum_{s=1}^{t^*-1} \prod_{r=0}^{s-1} \beta_{\mathrm{\ell}}(r)\right) \cdot  e^{-\lambda (g(X_0)-g(a))}$ 
for $t^* > 0$ and $X_0>a$. 

If additionally the set of states $S\cap \{x\mid x\le a\}$ is absorbing, then 
$\Prob(T_a < t^* \mid \filtzero) \le \left(\prod_{r=0}^{t^*-1} \beta_{\mathrm{\ell}}(r)\right) \cdot e^{-\lambda (g(X_0)-g(a))}$.
\end{enumerate}
\end{theorem}

Statement~$(ii)$ is also valid (but useless) 
if the expected hitting time is infinite. 
 \ref{sec:specialcasesvariable} studies specializations of
 these first two statements into existing variable and multiplicative
 drift theorems which are mostly concerned with expected hitting time.

\textbf{Special cases of (iii) and (iv).} If 
$E(e^{-\lambda (g(X_t)-g(X_{t+1}))} - \beta_{\mathrm{u}} \filtcond{ X_t > a}) \le 0$ for 
some time-independent $\beta_{\mathrm{u}}$, then 
Statement~$(iii)$ simplifies down to 
$\Prob(T_a>t^* \mid \filtzero) < \beta_{\mathrm{u}}^{t^*} \cdot  e^{\lambda (g(X_0)-g(a))}$; similarly 
for Statement~$(iv)$.

The proof of our main theorem is not too complicated. The tail bounds in 
$(iii)$ and $(iv)$ are obtained by the exponential method 
(a generalized Chernoff bound), which idea is also implicit in \cite{HajekDrift}.

\begin{proofof}{Proof of Theorem~\ref{theo:main}}
Since $g(X_t)=0$ iff $X_t=0$ and since the image of~$g$ is bounded from below by~$0$ and additionally by $\xmax$ in item~$(ii)$, 
the first two items follow from the classical additive drift theorem (Theorem~\ref{theo:additive}). To prove the 
third one, we consider the stopped process that does not move after time~$T_a$. We 
now use ideas implicit in \cite{HajekDrift} and argue that 
\begin{align*}
\Prob(T_a>t^* \mid \filtzero) & \;\le\; \Prob(X_{t^*}>a\mid \filtzero) \;\le\; \Prob(g(X_{t^*})> g(a)\mid \filtzero) \\
& \; =\; \Prob(e^{\lambda g(X_{t^*})} > e^{\lambda g(a)}\mid \filtzero) 
 \;<\; E(e^{\lambda g(X_{t^*}) - \lambda g(a)}\mid \filtzero),
\end{align*}
where the second inequality uses that $X_{t^*}>a$ implies $g(X_{t^*})>g(a)$, the equality
that $x\mapsto e^x$ is a bijection, 
and the last inequality is Markov's inequality. Now,
\begin{align*}
E(e^{\lambda g(X_{t^*})}\mid \filtzero) & = E(e^{\lambda g(X_{t^*-1})} \cdot E(e^{-\lambda (g(X_{t^*-1})-g(X_{t^*}))}\mid \mathcal{F}_{t^*-1}) \mid \filtzero)\\
& \le E(e^{\lambda g(X_{t^*-1})}  \mid \filtzero) \cdot\beta_{\mathrm{u}}(t^*-1) 
\end{align*}
using the prerequisite from the third item.
Unfolding the remaining expectation inductively    
(note that this does not 
assume independence of the differences $g(X_{r-1})-g(X_{r})$), 
we get 
\[
E(e^{\lambda g(X_{t^*})}\mid \filtzero) \le e^{\lambda g(X_0)} \prod_{r=0}^{t^*-1}  \beta_{\mathrm{u}}(r),
\]
altogether
\[
\Prob(T_a>t^*\mid \filtzero) < e^{\lambda (g(X_0)-g(a))} \prod_{r=0}^{t^*-1}  \beta_{\mathrm{u}}(r),
\]
which proves the third item.

The fourth item is proved similarly as the third one. Using a union bound and that $X_{t^*}\le a$ follows  
from $g(X_{t^*})\le g(a)$, 
\[
\Prob(T_a<t^*\mid \filtzero) \le \sum_{s=1}^{t^*-1} \Prob(g(X_s) \le g(a)\mid \filtzero) 
\]
for $t^*>0$, assuming $X_0 > a$. 
Moreover, 
\[
\Prob(g(X_s) \le g(a)\mid \filtzero) = \Prob(e^{-\lambda g(X_s)} \ge e^{-\lambda g(a)}\mid \filtzero)
\le E(e^{-\lambda g(X_s)+\lambda g(a) }\mid \filtzero)
\]
using again Markov's inequality. By the prerequisites, we get 
\[
E(e^{-\lambda g(X_s) }\mid \filtzero) \le e^{-\lambda g(X_0)} 
\prod_{r=0}^{s-1}  \beta_{\mathrm{\ell}}(r)
\]
Altogether,
\[
\Prob(T_a<t^* \mid \filtzero) \le \sum_{s=1}^{t^*-1} 
e^{-\lambda (g(X_0)+g(a))} \prod_{r=0}^{s-1}  \beta_{\mathrm{\ell}}(r).
\]

If furthermore $S\cap \{x\mid x\le a\}$ is absorbing then the event $X_{t^*}\le a$ is necessary for $T_a<t^*$. In this case,
\[
\Prob(T_a<t^*\mid \filtzero) \le \Prob(g(X_{t^*}) \le g(a)\mid \filtzero) 
\le  e^{-\lambda (g(X_0)+g(a))}
\prod_{r=0}^{t^*-1}  \beta_{\mathrm{\ell}}(r).
\]
\end{proofof}

Given some assumptions on the ``drift'' function $h$ that typically hold, 
Theorem~\ref{theo:main} can be simplified. The following corollary
will be used to prove the multiplicative drift theorem
(Theorem~\ref{theo:multiplicative-drift} in Section~\ref{sec:specialcasesother}).
 Some applications of it 
require 
a ``gap''  between optimal and non-optimal states, modelled 
by $\xmin>0$. One example is in fact multiplicative drift. 
 Another 
example is the process defined by $X_0\sim \text{Unif}[0,1]$ and $X_t=0$ for $t>0$. Its first 
hitting time of state~$0$ cannot be derived by drift arguments since 
the lower bound on the drift towards the optimum within the interval $[0,1]$ has limit~$0$.


\begin{corollary}
\label{cor:tailbound-h-convex-concave}
  Let $(X_t)_{t\in\N_0}$, be a stochastic process, adapted to a filtration $(\filt)_{t\in\N_0}$, over some state space  $S\subseteq \{0\}\cup \R_{\ge \xmin}$, where 
  $\xmin\ge 0$. Let  $h\colon \R_{\ge \xmin}\to\R^+$ be a function such
  that $1/h(x)$ is integrable on $ \R_{\ge \xmin}$ and $h(x)$ differentiable on $ \R_{\ge \xmin}$.
  Then the following statements hold for the first hitting time $T:=\min\{t\mid X_t=0\}$.
  \begin{enumerate}[leftmargin=!,labelwidth=7mm]
  \item[(i)] If $E(X_t-X_{t+1} - h(X_t) \filtcond{ X_t\ge \xmin}) \ge 0$ for all  $t\in\N_0$ and
    $\tfrac{\mathrm{d}}{\mathrm{d}x}h(x)\geq 0$, then 
$      E(T\mid \filtzero) \le \frac{\xmin}{h(\xmin)} + \int_{\xmin}^{X_0} \frac{1}{h(y)} \,\mathrm{d}y.$
		
  \item[(ii)] If $E(X_t-X_{t+1} - h(X_t) \filtcond{ X_t\ge \xmin}) \le 0$, $X_t\le \xmax$ for some $\xmax>0$ and all $t\in\N_0$,  and
    $\tfrac{\mathrm{d}}{\mathrm{d}x}h(x)\leq 0$, then 
$      E(T\mid \filtzero) \ge \frac{\xmin}{h(\xmin)} + \int_{\xmin}^{X_0} \frac{1}{h(y)} \,\mathrm{d}y.$
		
  \item[(iii)] If $E(X_t-X_{t+1} - h(X_t) \filtcond{ X_t\ge \xmin}) \ge 0$  for all  $t\in\N_0$ and 
    $\tfrac{\mathrm{d}}{\mathrm{d}x}h(x)\geq \lambda$ for some $\lambda>0$, then
$      \Prob(T>t^* \mid \filtzero) < \exp\left(-\lambda \left(t^*-\frac{\xmin}{h(\xmin)}-\int_{\xmin}^{X_0} \frac{1}{h(y)}\,\mathrm{d}y\right)\right).$
  \item[(iv)]
    If $E(X_t-X_{t+1} - h(X_t) \filtcond{ X_t\ge \xmin}) \le 0$
    for all  $t\in\N_0$ 
    and $\tfrac{\mathrm{d}}{\mathrm{d}x}h(x)\leq -\lambda$ for  some $\lambda>0$, then, on $X_0>0$, 
$      \Prob(T < t^* \mid \filtzero) < \frac{e^{\lambda t^*}-e^{\lambda}}{e^\lambda-1} \exp\left(-\frac{\lambda\xmin}{h(\xmin)}-\int_{\xmin}^{X_0} \frac{\lambda}{h(y)}\,\mathrm{d}y\right).
$ \end{enumerate}
\end{corollary}

\begin{proof}
  As in Remark~\ref{remark:function-h}, let $g(x) := \xmin/h(\xmin) + \int_{\xmin}^x 1/h(y) \,\mathrm{d}y$ and 
	$g(0):=0$.
  Note that for the second derivative we have $g''(x)=-(\tfrac{\mathrm{d}}{\mathrm{d}x}h(x))/h(x)^2$.

  For $(i)$, it suffices to show that condition $(i)$ of
  Theorem~\ref{theo:main} is satisfied for $\alpha_u:=1$.
  From the assumption $h'(x)\geq 0$, it follows that $g''(x)\leq 0$, 
  hence $g$ is a concave function. Jensen's inequality therefore
  implies that
  \begin{align*}
    E(g(X_t)-g(X_{t+1})  \filtcond{ X_t\ge \xmin})
    & \geq g(X_t) - g(    E(X_{t+1} \filtcond{ X_t\ge \xmin}))\\
    & \geq \int_{X_t-h(X_t)}^{X_t}\frac{1}{h(y)}\,\mathrm{d}y
     \geq \frac{1}{h(X_t)}\cdot h(X_t) = 1,
  \end{align*}
  where the last inequality holds because $h$ is a non-decreasing function.

  For $(ii)$, we note that $g(X_t)\le b$ for some $b\in \R$ since $X_t\le \xmax$ for $t\in\N_0$. Hence, 
	it suffices to show that condition $(ii)$ of
  Theorem~\ref{theo:main} is satisfied for $\alpha_\ell:=1$.
  From the assumption $h'(x)\leq 0$, it follows that $g''(x)\geq 0$,
  hence $g$ is a convex function. Jensen's inequality therefore
  implies that
  \begin{align*}
    E(g(X_t)-g(X_{t+1})  \filtcond{ X_t\ge \xmin})
    & \leq g(X_t) - g(    E(X_{t+1} \filtcond{ X_t\ge \xmin}))\\
    &  \leq \int_{X_t-h(X_t)}^{X_t}\frac{1}{h(y)}\,\mathrm{d}y
      \leq \frac{1}{h(X_t)}\cdot h(X_t) = 1,
  \end{align*}
  where the last inequality holds because $h$ is a non-increasing
  function.

  For $(iii)$, it suffices to show that condition $(iii)$ of
  Theorem~\ref{theo:main} is satisfied for $\beta_u := e^{-\lambda}$. 
  Let $f_1(x):=e^{\lambda g(x) }$ and note that 
      $f_1''(x)=\frac{\lambda e^{\lambda g(x)}}{h(x)^2}\cdot (\lambda-h'(x))$.
  Since $h'(x)\geq \lambda$, it follows that $f_1''(x) \leq 0$ and $f_1$
  is a concave function. By Jensen's inequality, it holds that
  \begin{align*}
    E(e^{-\lambda (g(X_t)-g(X_{t+1}))}\filtcond{ X_t\ge \xmin}) 
    \leq 
    e^{-\lambda r}
  \end{align*}
  where
  \begin{align*}
    r & :=    g(X_t)-g(\E{X_{t+1}\filtcond{ X_t\ge \xmin}})
       \geq \int_{X_t-h(X_t)}^{X_t}\frac{1}{h(y)}\,\mathrm{d}y
       > \frac{1}{h(X_t)}\cdot h(X_t)=1,
  \end{align*}
  where the last inequality holds because $h$ is 
  monotone increasing.

  For $(iv)$, it suffices to show that condition $(iv)$ of
  Theorem~\ref{theo:main} is satisfied for $\beta_\ell :=
  e^{\lambda}$. 
  Let $f_2(x):=e^{-\lambda g(x) }$ and note that 
      $f_2''(x)=\frac{\lambda e^{-\lambda g(x)}}{h(x)^2}\cdot (\lambda+h'(x))$.
  Since $h'(x)\leq \lambda$, it follows that $f_2''(x) \leq 0$ and $f_1$
  is a concave function.   
  By Jensen's inequality, it holds that
  \begin{align*}
    E(e^{\lambda (g(X_t)-g(X_{t+1}))}\filtcond{ X_t\ge \xmin}) 
    \leq 
    e^{\lambda r}
  \end{align*}
  where 
  \begin{align*}
    r & :=    g(X_t)-g(\E{X_{t+1}\filtcond{ X_t\ge \xmin}})
       \leq \int_{X_t-h(X_t)}^{X_t}\frac{1}{h(y)}\,\mathrm{d}y
       < \frac{1}{h(X_t)}\cdot h(X_t)=1,
  \end{align*}
  where the last inequality holds because $h$ is 
  monotone decreasing.
\end{proof}


Condition $(iii)$ and $(iv)$ of Theorem~\ref{theo:main} involve an
mgf., which may be tedious to compute.  Inspired by \cite{HajekDrift}
and \cite{Lehre12DriftTutorial}, we show that bounds on the mgfs.\
follow from more user-friendly conditions based on stochastic
dominance between random variables, here denoted by $\prec$.

\begin{theorem}
\label{theo:main-simplifiedexponential}
Let $(X_t)_{t\in\N_0}$, be a stochastic process, adapted to a filtration $(\filt)_{t\in\N_0}$, over some state space  $S\subseteq \{0\}\cup \R_{\ge \xmin}$, where 
  $\xmin\ge 0$. Let  $h\colon  \R_{\ge \xmin}\to\R^+$ be a function such
  that $1/h(x)$ is integrable on $\R_{\ge \xmin}$. 
	Suppose there exist a random variable~$Z$ and some $\lambda>0$  such that 
$\lvert \int_{X_{t+1}}^{X_t} 1/h(x)\,\mathrm{d}x\rvert \prec Z$ for  
$X_{t}\ge \xmin$ for all $t\in\N_0$ and $E(e^{\lambda Z}) = D$ for some $D>0$. 
 Then the 
following two statements hold for the first hitting time $T:=\min\{t\mid X_t=0\}$.

\begin{enumerate}[leftmargin=!,labelwidth=7mm]
\item[(i)] 
If  $E(X_t-X_{t+1} - h(X_t)\filtcond{ X_t\ge \xmin}) \ge 0$  for all $t\in\N_0$
then for any $\delta>0$, and $\eta:=\min\{\lambda, \delta\lambda^2/(D-1-\lambda)\}$ and $t^*>0$ 
it holds that
\[
\Prob(T>t^* \mid \filtzero) \le \exp\left(\eta \left(\int_{\xmin}^{X_0} 1/h(x) \,\mathrm{d}x-(1-\delta)t^*\right)\right).
\]
\item[(ii)]
If
$E(X_t-X_{t+1} - h(X_t) \filtcond{ X_t\ge \xmin}) \le 0$ for all $t\in \N_0$ 
then for any $\delta>0$, $\eta:=\min\{\lambda, \delta\lambda^2/(D-1-\lambda)\}$ and $t^*>0$ 
it holds on $X_0>0$ that 
\[
\Prob(T < t^* \mid \filtzero) \le \exp\left(\eta \left((1+\delta)t^* - \int_{\xmin}^{X_0} 1/h(x) \,\mathrm{d}x\right)\right)\frac{1}{\eta(1+\delta)}.
\]
If state~$0$ is absorbing then
$\Prob(T < t^* \mid \filtzero) \le \exp\left(\eta ((1+\delta)t^* - \int_{\xmin}^{X_0} 1/h(x) \,\mathrm{d}x)\right).$
\end{enumerate}
\end{theorem}

\begin{remark}Theorem~\ref{theo:main-simplifiedexponential} assumes a stochastic dominance of the kind  
$\lvert \int_{X_{t+1}}^{X_t} 1/h(x)\,\mathrm{d}x\rvert \prec Z$. This is implied 
by 
$\lvert X_{t+1}-{X_t}\rvert  (1/\!\inf_{x\ge \xmin} h(x))  \prec Z$.
\end{remark}

\begin{proof}
As in Remark~\ref{remark:function-h}, let 
$g(x) := \frac{\xmin}{h(\xmin)} + \int_{\xmin}^x \frac{1}{h(y)} \,\mathrm{d}y$ for $x\ge \xmin$ 
and $g(0):=0$. 
Let $\Delta_t:=g(X_t)-g(X_{t+1})$ and note that 
$\Delta_t = \int_{X_{t+1}}^{X_t} \frac{1}{h(x)}\,\mathrm{d}x$. 
To satisfy the third condition of Theorem~\ref{theo:main}, we note 
\begin{align*}
 E(e^{-\eta \Delta_t}) & = 1-\eta E(\Delta_t) + \sum_{k=2}^\infty \frac{\eta^k E(\Delta_t^k)}{k!}
\le 
1-\eta E(\Delta_t) + \eta^2 \sum_{k=2}^\infty \frac{\eta^{k-2}  E(\lvert \Delta_t\rvert ^k)}{k!}\\
& 
\le 1-\eta E(\Delta_t) + \eta^2 \sum_{k=2}^\infty \frac{\lambda^{k-2}  E(\lvert \Delta_t\rvert ^k)}{k!} 
= 1-\eta + \frac{\eta^2}{\lambda^2} (e^{\lambda Z}-\lambda E(Z)-1),
\end{align*}
where we have used $E(\Delta_t)\ge 1$ (proved in Theorem~\ref{theo:main}) and $\lambda\ge \eta$ . 
Since $\lvert\Delta_t\rvert \prec Z$, also $E(Z)\ge 1$. Using $e^{\lambda Z} = D$ and $\eta\le \delta\lambda^2/(D-1-\lambda)$, 
we obtain 
\[
E(e^{-\eta \Delta_t})
\le 1- \eta + \delta \eta = 1- (1-\delta) \eta \le e^{-\eta (1-\delta)}.
\]
Setting $\beta_\mathrm{u}:=e^{-\eta(1-\delta)}$ and using $\eta$ as the $\lambda$ 
of Theorem~\ref{theo:main} proves the first statement.

For the second statement, analogous calculations prove
\[
E(e^{\eta \Delta_t}) \le 1 + (1+\delta)\eta \le e^{\eta (1+\delta)}.
\]
We set  
$\beta_\mathrm{\ell}:=e^{\eta(1+\delta)}$, use $\eta$ as the $\lambda$ 
of Theorem~\ref{theo:main}.$(iv)$ and note that 
\[
\frac{e^{\lambda(1+\delta) t^*}-e^{\lambda(1+\delta)}}{
e^{\lambda(1+\delta)} -1 } \le \frac{e^{\lambda(1+\delta) t^*}}{\lambda(1+\delta)},
\]
which was to be proven.  If additionally an absorbing state~$0$ is
assumed, the stronger upper bound follows from the corresponding
statement in Theorem~\ref{theo:main}.$(iv)$.
\end{proof}

\section{Applications of the Tail Bounds}
\label{sec:applyingtail}
So far we have mostly derived bounds on the expected first hitting 
time using Statements~$(i)$ and~$(ii)$ of our general drift theorem. 
As our main contribution, we show that the general drift theorem
(Theorem~\ref{theo:main}), together with the function $g$ defined
explicitly in Remark~\ref{remark:function-h} in terms of the one-step
drift, constitute a very general and precise tool for analysis of
stochastic processes. In particular, it provides very sharp tail
bounds on the running time of randomized search heuristics which were
not obtained before by drift analysis. 
It also provides tail
bounds on random recursions, such as those in analysis of
random permutations (see Section~\ref{sec:recurrences}).

As already noted, virtually all existing drift
theorems, including an existing result proving tail bounds with multiplicative drift, 
can be phrased as special cases of the general drift theorem
(see \ref{sec:specialcasesvariable}). 
Recently, in \cite{KoetzingAlgo16} 
different tail bounds were proven for the scenario of additive drift using 
classical concentration inequalities such as Azuma-Hoeffding bounds. 
These bounds are not directly comparable to the ones from 
our general drift theorem; 
they are more specific but yield even stronger exponential bounds.

We first give sharp tail bounds on the optimization time of
the \ea 
which maximizes pseudo-Boolean functions 
$f\colon\{0,1\}^n\to\R$. The optimization time is defined in the canonical way at the smallest $t$ such 
that $x_t$ is an optimum. We consider classical benchmark problems from the theory
of RSHs. Despite their simplicity, their
analysis has turned out surprisingly difficult and research is still ongoing.

\begin{algorithm}
\caption{(1+1) Evolutionary Algorithm (EA)}
\label{alg:oneoneea}

\begin{algorithmic}
 \STATE Choose uniformly at random an initial bit string $x_0 \in \{0,1\}^n$.
 \FOR{$t:=0$ {\bf to} $\infty$}
 \STATE Create $x'$ by flipping each bit in $x_t$ independently with probability $1/n$ \emph{(mutation)}.
 \STATE $x_{t+1}:=x'$ if $f(x') \ge f(x_t)$, and $x_{t+1}:=x_t$ otherwise \emph{(selection)}. 
 \ENDFOR
\end{algorithmic}
\end{algorithm}
 

\subsection{OneMax, Linear Functions and LeadingOnes}

A simple pseudo-Boolean function is given by 
$\OneMax(x_1,\dots,x_n)=x_1+\dots+x_n$. It is included in the class of so-called linear functions  
$f(x_1,\dots,x_n)=w_1x_n+\dots+w_n x_n$, where $w_i\in\R$ for $1\le i\le n$.  
We start by deriving very precise bounds on first the expected optimization time 
of the \ea on \OneMax and then prove tail bounds.  The lower bounds obtained will imply 
results for all linear functions. Note that in \cite{DFWVariable}, 
already the following result has been proved using variable drift analysis.

\begin{statement}\textbf{\upshape Theorems~3 and~5 in \cite{DFWVariable}} 
The expected optimization time of the \ea on \OneMax is at most
 $en\ln n - c_1 n + O(1)$ 
and at least $en \ln n - c_2n$ for certain constants~$c_1,c_2>0$.
\end{statement}

The constant~$c_2$ is not made explicit in \cite{DFWVariable}, whereas the constant~$c_1$ is stated as $0.369$. However, unfortunately 
this value is due to a typo in the very last line of the proof -- $c_1$ should have been 0.1369 instead. We correct this 
mistake in a self-contained proof. Furthermore, we improve the lower bound using variable drift. To this end, we use the following bound on the drift.

\begin{lemma}
\label{lem:bound-drift-onemax}
Let $X_t$ denote the number of zeros of the current search point of the \ea on \OneMax. Then 
\[
\left(1-\frac{1}{n}\right)^{n-x}\frac{x}{n}
\le 
E(X_t-X_{t+1}\mid X_t=x) \le 
\left(\left(1-\frac{1}{n}\right)\left(1+\frac{x}{(n-1)^2}\right)\right)^{n-x} 
\frac{x}{n}.
\]
\end{lemma}

\begin{proof}
The lower bound considers the expected number of flipping zero-bits, assuming that 
no one-bit flips. 
The upper bound is obtained in the proof of Lemma~6 in \cite{DFWVariable} and denoted 
by $S_1\cdot S_2$, but is not made explicit in the statement of the lemma.
\end{proof}

\begin{theorem}
\label{theo:expected-onemax-upper-lower}
The expected optimization time of the \ea on \OneMax is at most $en\ln n - 0.1369n + O(1)$ 
and at least $en \ln n - 7.81791 n - O(\log n)$.
\end{theorem}

\begin{proof}
Note that with probability $1-2^{-\Omega(n)}$ we have 
$\tfrac{(1-\epsilon)n}{2} \le X_0 \le \tfrac{(1+\epsilon)n}{2}$ 
for an arbitrary constant~$\epsilon>0$. Hereinafter, we assume this event to happen, which 
only adds an error term of absolute value 
$2^{-\Omega(n)}\cdot n\log n=2^{-\Omega(n)}$ to the expected optimization time.

In order to apply the variable drift theorem (more precisely, Theorem~\ref{theo:variable-rowe-sudholt} for 
the upper and Theorem~\ref{theo:variable-dfw} for the lower bound), 
we manipulate and estimate the expressions from Lemma~\ref{lem:bound-drift-onemax} 
to make them easy to integrate. To prove the upper bound on the optimization time, 
we observe
\begin{align*}
E(X_t-X_{t+1} \mid X_t=x) & \;\ge\; \left(1-\frac{1}{n}\right)^{n-x}\frac{x}{n} \\
&  \;=\; 
\left(1-\frac{1}{n}\right)^{n-1} \cdot \left(1-\frac{1}{n}\right)^{-x}  
\cdot 
  \frac{x}{n} \cdot  \left(1-\frac{1}{n}\right) \\
	& \;\ge\; e^{-1+\tfrac{x}{n}} \cdot \frac{x}{n} \cdot  \left(1-\frac{1}{n}\right) =:h_\ell(x).
\end{align*}

Now, by the variable drift theorem, the optimization time~$T$ satisfies
\begin{align*}
\E{T\mid X_0} & \le \frac{1}{h_\ell(1)} + \int_{1}^{(1+\epsilon)n/2} \frac{1}{h_\ell(x)} \,\mathrm{d}x 
\le \left(en + \int_{1}^{(1+\epsilon)n/2} 
e^{1-\frac{x}{n}} \cdot \frac{n}{x} \right)\left(1-\frac{1}{n}\right)^{-1}\\ 
&
\le \left(en - en \left[E_1(x/n)\right]_{1}^{(1+\epsilon)n/2}\right)\left(1+\bigO{\frac{1}{n}}\right),
\end{align*}
where $E_1(x):=\int_{x}^\infty \frac{e^{-t}}{t}\,\mathrm{d}t$ denotes the exponential integral (for $x>0$). The latter is 
estimated using  
the series representation $E_1(x)=-\ln x - \gamma - \sum_{k=1}^{\infty} \tfrac{(-x)^{k} }{kk!}$, with $\gamma = 0.577\dots$ 
being the Euler-Mascheroni constant (see Equation~5.1.11 in \cite{AbramotitzStegun}).
We get 
for sufficiently small~$\epsilon$ that 
\[
-\left[E_1(x/n)\right]_{1}^{(1+\epsilon)n/2} = 
 E_1(1/n) -E_1((1+\epsilon)/2)    \le -\ln (1/n) - \gamma + O(1/n) - 0.559774.
\]
Altogether,
\[
\E{T\mid X_0} \le en\ln n + en(1-0.559774 -\gamma)  + O(\log n)
\le en\ln n - 0.1369n + O(\log n) 
\]
which proves the upper bound.

For the lower bound on the optimization time, we need according to Theorem~\ref{theo:variable-dfw} 
a monotone process (which is satisfied) and a function $c$ bounding the progress towards 
the optimum. We use $c(x)=x-\log x-1$. Since each bit flips with probability~$1/n$, we get  
\begin{align*}
\Prob(X_{t+1}  \le X_t - \log(X_t)-1) & \;\le\; \binom{X_t}{\log(X_t)+1} \left(\frac{1}{n}\right)^{\log(X_t)+1} 
 \;\le\; \left(\frac{eX_t}{n\log(X_t)+n}\right)^{\log(X_t)+1}.
\end{align*}
The last bound takes its maximum at $X_t=2$ within the interval $[2,\dots,n]$ and is 
$O(n^{-2})$ then. For $X_t=1$, we trivially have $X_{t+1} \ge c(X_t)=0$. 
Hence, by assuming $X_{t+1}\ge c(X_t)$ for all $t=O(n\log n)$, we only introduce 
an additive error of value $O(\log n)$. 

Next the upper bound on the drift from Lemma~\ref{lem:bound-drift-onemax} is 
manipulated. We get for some sufficiently large constant~$c^*>0$ that 
\begin{align*}
 E(X_t-X_{t+1} \mid X_t=x) & \le \left(\left(1-\frac{1}{n}\right)\left(1+\frac{x}{(n-1)^2}\right)\right)^{n-x} \cdot \frac{x}{n} \\ 
& \le e^{-1+\frac{x}{n}+\frac{x(n-x)}{n^2}} \cdot \frac{x}{n} \cdot \left(\frac{1+x/(n-1)^2}{1+x/(n^2)}\right)^{n-x} \\ 
& \le e^{-1+\frac{2x}{n}} \cdot \frac{x}{n} \cdot \left(1+{\frac{c^*}{n}}\right) =: h^*(x),
\end{align*}
where we used $1+x\le e^x$ twice. The drift theorem requires a function $h_{\mathrm{u}}(x)$ such that 
$h^*(x)\le h_\mathrm{u}(c(x)) = h_{\mathrm{u}}(x-\log x-1)$. Introducing the substitution 
$y:=y(x):=x-\log x-1$ and its inverse function 
$x(y)$, we 
choose $h_{\mathrm{u}}(y):=h^*(x(y))$.

 We obtain
\begin{align*}
& \E{T\mid X_0} \ge \left(\frac{1}{h^*(x(1))} + \int_{1}^{(1-\epsilon)n/2} \frac{1}{h^*(x(y))} \,\mathrm{d}y \right)
\left(1-\bigO{\frac{1}{n}}\right) \\
& \ge \left(\frac{1}{h^*(4)} + 
\int_{x(1)}^{x((1-\epsilon)n/2)} \frac{1}{h^*(x)} \left(1-\frac{1}{x}\right) \,\mathrm{d}x \right) \left(1-\bigO{\frac{1}{n}}\right)\\
& \ge 
\left(\frac{en}{4} + 
\int_{2}^{(1-\epsilon)n/2} 
e^{1-\frac{2x}{n}} \cdot \frac{n}{x} 
 \left(1-\frac{1}{x}\right) \,\mathrm{d}x \right) \left(1-\bigO{\frac{1}{n}}\right)\\
& = 
\left(\frac{en}{4} + 
\int_{2}^{(1-\epsilon)n/2} 
e^{1-\frac{2x}{n}} \cdot \frac{n}{x} 
  \,\mathrm{d}x 
	- 
	\int_{2}^{(1-\epsilon)n/2} 
e^{1-\frac{2x}{n}} \cdot \frac{n}{x^2} 
  \,\mathrm{d}x \right) \left(1-\bigO{\frac{1}{n}}\right)
\end{align*}
where the second inequality uses integration by substitution  and $x(1)=4$, the third one 
$x(y)\le y$, 
and the last one partial integration.

With respect to the first integral in the last bound, the only difference compared to the upper bound  
is the~$2$ in the exponent of $e^{-1+\frac{2x}{n}}$, such that we can proceed analogously to the above and 
obtain $-en E_1(2x/n)+C$ as anti-derivative. The anti-derivative 
of the second integral is $2e E_1(2x/n) - e^{1-2x/n} \tfrac{n}{x} + C$. 

We obtain
\begin{align*}
& \E{T\mid X_0} \ge 
\left(\frac{en}{4} + \left[-(2e+en)E_1(2x/n)+e^{1-2x/n} \frac{n}{x}\right]_{2}^{(1-\epsilon)n/2}\right)\left(1-\bigO{\frac{1}{n}}\right) 
\end{align*}
Now, for sufficiently small~$\epsilon$, 
\[
-\left[E_1(2x/n)\right]_{2}^{(1-\epsilon)n/2} \ge 
-\ln (8/n) - \gamma - O(1/n) - 0.21939 \ge \ln n -2.76048 - O(1/n)
\]
and
\[
\left[e^{1-2x/n} \frac{n}{x}\right]_{2}^{(1-\epsilon)n/2} \ge 1.9999 - \frac{en}{4} - O(1/n).
\]
Altogether,
\[
\E{T\mid X_0}
\ge en\ln n - 7.81791n - O(\log n)
\]
as suggested.
\end{proof}

Knowing the expected optimization time precisely, we  now turn to our main new contribution, 
\ie, to derive sharp bounds. Note that the following upper concentration inequality in Theorem~\ref{theo:tails-onemax} 
is not new but is already implicit in the work on multiplicative drift analysis 
by \cite{DJWMultiplicativeAlgorithmica}. In fact, a very similar upper bound is even 
available for all linear functions \cite{WittCPC13}. By contrast, the lower concentration inequality 
is a novel and non-trivial result.

\begin{theorem}
\label{theo:tails-onemax}
The optimization time of the \ea on \OneMax is at least 
$en\ln n - c n - ren$, where $c$ is a constant, 
with probability at least $1-e^{-r/2}$ for any $r\ge 0$. It is at most 
$en\ln n + r e n$ 
with probability at least $1-e^{-r}$.
\end{theorem}

\begin{proofof}{Proof of Theorem 4, upper tail.}
This tail can be easily derived 
from the multiplicative drift theorem 
(Theorem~\ref{theo:multiplicative-drift}). Let $X_t$ denote the number of zeros at time~$t$.
 By Lemma~\ref{lem:bound-drift-onemax}, we can choose 
$\delta:=1/(en)$. Then the upper bound follows since $X_0\le n$ and $\xmin=1$. 
\end{proofof}

We only consider the lower tail. The aim is to prove it using Theorem~\ref{theo:main}.$(iv)$, which 
includes a bound on the moment-generating function of the drift of $g$. We first set up the $h$ (and thereby 
the $g$) used for our purposes. Obviously, $\xmin:=1$. 

\begin{lemma}
\label{lem:lowertailonemax-one}
Consider the \ea on \OneMax and let the random variable $X_t$ 
denote the current number of zeros at time~$t\ge 0$. Then  
$h(x) := \exp\left(-1+2\lceil x\rceil/n\right) \cdot (\lceil x\rceil/n) \cdot \left(1+c^*/n\right),$
where $c^*>0$ is a sufficiently large constant, 
satisfies the condition 
$E(X_t-X_{t+1} \mid X_t=i) \le h(i)$ for $i\in\{1,\dots,n\}$. 
Moreover, define 
$g(i):=\xmin/h(\xmin) + \int_{\xmin}^i 1/h(y)\,\mathrm{d}{y}$ and 
$\Delta_t:=g(X_t)-g(X_{t+1})$. Then 
$g(i) = \sum_{j=1}^i 1/h(j)$ and 
$\Delta_t \le \sum_{j=X_{t+1}+1}^{X_t} e^{1-2X_{t+1}/n} \cdot (n/j).$
\end{lemma}

\begin{proof}
According to Lemma~\ref{lem:bound-drift-onemax}, 
$h^*(x):=((1-\tfrac{1}{n})(1+\tfrac{x}{(n-1)^2})^{n-x} \frac{x}{n}$ is an upper bound on the drift. 
We obtain $h(x)\ge h^*(x)$ using the simple estimations exposed in the proof 
of Theorem~\ref{theo:expected-onemax-upper-lower}, lower bound part.

The representation of $g(i)$ as a sum 
follows immediately from $h$ due to the ceilings. The bound on $\Delta_t$ 
follows from~$h$ by estimating 
$e^{-1+\frac{2\lceil x\rceil }{n}} \cdot \left(1+{\frac{c^*}{n}}\right)\ge e^{-1+2x/n}$. 
\end{proof}

The next lemma provides a bound on the mgf.\ 
of the drift of $g$, which will depend on the current state. Later, the 
state will be estimated based on the current point of time, leading
to a time-dependent bound on the mgf. 
Note that we do not need the whole natural 
filtration based on $X_0,\dots,X_t$ but 
only~$X_t$ since we are dealing with a Markov chain.

\begin{lemma}
\label{lem:mgf-onemax}
Let $\lambda:=1/(en)$ and $i\in\{1,\dots,n\}$. Then
$E(e^{\lambda \Delta_t}\mid X_t=i) \;\le\; 1+ \lambda + 2\lambda/i + \littleo{\lambda/\!\log n}$.
\end{lemma}

\begin{proof}
We distinguish between three major cases. 

\textbf{Case 1:} $i=1$. Then $X_{t+1}=0$, implying $\Delta_t\le en$, with probability $(1/n)(1-1/n)^{n-1} = (1/(en))(1+1/(n-1))$ and 
$X_{t+1}=i$ otherwise. We get
\begin{align*}
E(e^{\lambda \Delta_t}\mid  X_t=i) & \;\le\;  \frac{1}{en} \cdot e^1 + \left(1-\frac{1}{en}\right) + \bigO{\frac{1}{n^2}} \\
& \;\le\; 
1 + \frac{e-1}{en} + \bigO{\frac{1}{n^2}} \;\le\;  1 + \lambda + \frac{(e-2)\lambda}{i} + \littleo{\frac{\lambda}{\ln n}}.
\end{align*}

\textbf{Case 2:} $2\le i\le \ln^3 n$. Let $Y:=i-X_{t+1}$ and note that  
$\Prob(Y \ge 2) \le (\ln^6 n)/n^2$ since the probability 
of flipping a zero-bit is at most $(\ln^3 n)/n$. 
We further subdivide the case according to whether $Y\ge 2$ or not. 

\textbf{Case 2a:} $2\le i\le \ln^3 n$ and $Y\ge 2$. 
The largest value of $\Delta_t$ is taken when $Y=i$. 
Using Lemma~\ref{lem:lowertailonemax-one} and estimating the $i$-th Harmonic number, we have 
$\lambda\Delta_t \le (\ln i) + 1 \le 3(\ln\ln n) +1$. 
The contribution to the mgf.\ is bounded by 
\[
E(e^{\lambda \Delta_t}\cdot\indic{X_{t+1}\le i-2}\mid  X_t=i) \;\le\; e^{3\ln\ln n + 1} \cdot \left(\frac{\ln^6 n}{n^2} \right) 
= \littleo{\frac{\lambda}{\ln n}}. 
\]

\textbf{Case 2b:} $2\le i\le \ln^3 n$ and $Y<2$. 
Then $X_{t+1}\ge X_t-1$, which implies  
$\Delta_t \le en(\ln(X_t)-\ln(X_{t+1}))$. We obtain
\begin{align*}
& E(e^{\lambda \Delta_t}\cdot\indic{X_{t+1}\ge i-1}\mid  X_t=i)
\;\le\;
E(e^{\ln (i/X_{t+1})}) \;\le\; 
E(e^{\ln(1+\frac{i-X_{t+1}}{i-1})}) \\
& \;=\; \expec{1+\frac{Y}{i-1}},
\end{align*}
where the first inequality estimated $\sum_{i=j+1}^k \tfrac{1}{i}\le \ln(k/j)$ and 
the second one used $X_{t+1}\ge i-1$. From Lemma~\ref{lem:bound-drift-onemax}, 
we get $E(Y) \le \frac{i}{en}(1+O((\ln^3 n)/n))$ for $i\le \ln^3 n$.  
This implies
\begin{align*}
& \expec{1+\frac{i-X_{t+1}}{i-1}} \le 1 + \frac{i}{en(i-1)} \left(1+\bigO{\frac{\ln^3 n}{n}}\right) \\
& = 1 + \frac{1}{en}\cdot \left(1+\frac{1}{i-1}\right)  \left(1+\bigO{\frac{\ln^3 n}{n}}\right)
 = 
1 + \lambda + \frac{2\lambda}{i} + \littleo{\frac{\lambda}{\ln n}},
\end{align*}
using $i/(i-1)\le 2$ in the last step. 
Adding the bounds from the two sub-cases 
proves the lemma in Case~2.

\textbf{Case 3:} $i>\ln^3 n$. 
Note that 
$\Prob(Y\ge \ln n)\le \binom{n}{\ln n}\left(\tfrac{1}{n}\right)^{\ln n} \le 1/(\ln n)!$. 
We further subdivide the case according to whether $Y\ge \ln n$ or not. 

\textbf{Case 3a:} $i>\ln^3  n$ and $Y\ge \ln n$. 
Since $\Delta_t\le en (\ln n+1)$, we get
\[
E(e^{\lambda \Delta_t}\cdot \indic{X_{t+1}\le i-\ln^3 n}\mid X_t = i) 
\le \frac{1}{(\ln n)!} \cdot e^{\ln n+1}  = \littleo{\frac{\lambda}{\ln n}}
\]

\textbf{Case 3b:} $i>\ln^3  n$ and $Y< \ln n$. 
Then, using Lemma~\ref{lem:lowertailonemax-one} and proceeding similarly as in Case~2b,  
\begin{align*}
& E(e^{\lambda \Delta_t}\cdot \indic{X_{t+1}> i-\ln n}\mid X_t = i) \\
& \le
E(e^{\lambda \exp(1-2(i-\ln n)/n) \cdot n  \ln(i/X_{t+1}) }
\mid X_t = i) = \expec{\left(1+\frac{i-X_{t+1}}{X_{t+1}}\right)^{\exp((-2i+\ln n)/n)}}.
\end{align*}
Using $i>\ln^3 n$ and Jensen's inequality, the last expectation is at most
\begin{align*}
& \left(1+\expec{\frac{i-X_{t+1}}{X_{t+1}}}\right)^{\exp((-2i+\ln n)/n)}
\le 
 \left(1 + \E{\frac{Y}{i-\ln n}}\right)^{\exp((-2i+\ln n)/n)} \\
& \le  \left(1 + \E{\frac{Y}{i(1-1/\!\ln^2 n)}}\right)^{\exp((-2i+\ln n)/n)},
\end{align*}
where the last inequality used again $i>\ln^3 n$. 
Since $E(Y)\le e^{-1+2i/n} \frac{i}{n} (1+c^*/n)$, we conclude
\begin{align*}
& E(e^{\lambda \Delta_t}\cdot \indic{X_{t+1}> i-\ln n}\mid X_t = i)
\le \left(1 + \frac{e^{2i/n}}{en(1-1/\!\ln^2 n)}\right)^{\exp((-2i+\ln n)/n)}\\
&  \le \left(1 + \frac{1}{en(1-1/\!\ln^2 n)}\right) \left(1+\bigO{\frac{\ln n}{n^2}}\right)
\le 
1 + \lambda + \littleo{\frac{\lambda}{\ln n}},
\end{align*}
where we used $(1+ax)^{1/a} \le 1+x$ for $x\ge 0$ and $a\ge 1$. Adding up the bounds 
from the two sub-cases, we have proved the lemma in Case~3.

Altogether,
\[
E(e^{\lambda \Delta_t}\mid X_t=i) 
\le 
1 + \lambda + \frac{2\lambda}{i} + \littleo{\frac{\lambda}{\ln n}}.
\]
for all $i\in\{1,\dots,n\}$.
\end{proof}

The bound on the mgf.\ of $\Delta_t$ derived in Lemma~\ref{lem:mgf-onemax} is particularly large for $i=O(1)$, \ie, 
if the current state $X_t$ is small. If 
$X_t=O(1)$ held during the whole optimization process, we could not prove the lower tail in Theorem~\ref{theo:tails-onemax} 
from the lemma. However, it is easy to see that $X_t=i$ only holds for an expected number 
of at most $en/i$ steps. Hence, most of the time the term $2\lambda/i$ is negligible, 
and the time-dependent $\beta_{\mathrm{\ell}}(t)$-term from Theorem~\ref{theo:main}.$(iv)$ comes into play. We make 
this precise in the following proof, where we iteratively bound the probability of the process 
being at ``small'' states.

\begin{proofof}{Proof of Theorem~\ref{theo:tails-onemax}, lower tail.}
With overwhelming probability $1-2^{-\Omega(n)}$, $X_0\ge (1-\epsilon)n/2$ for an 
arbitrarily small constant~$\epsilon>0$, which we assume to happen. 
We consider phases in the optimization process. Phase~$1$ starts with initialization and 
ends before the first step where $X_t< e^{\frac{\ln n - 1}{2}}=\sqrt{n}\cdot e^{-1/2}$. Phase~$i$, where~$i>1$, follows Phase~$i-1$ and 
ends before the first step where $X_t< \sqrt{n}\cdot e^{-i/2}$. Obviously, the optimum is not found 
before the end of Phase $\ln(n)$; however, this does not tell us anything about the optimization time yet.

We say that Phase~$i$ is \emph{typical} if it does not end before 
time~$eni-1$. We will prove inductively that the probability of one of the first $i$ phases not being 
typical is at most~$c'e^{\frac{i}{2}}/\!\sqrt{n} = c'e^{\frac{i-\ln n}{2}}$ for some constant~$c'>0$. 
This implies the theorem since an optimization time of 
at least $en\ln n - cn - ren$ is implied by the event that Phase $\ln n -\lceil r-c/e\rceil$ is typical, which 
has probability at least $1-c'e^{\frac{-r+c/e+1}{2}} = 1-e^{\frac{-r}{2}}$ for $c=e(2\ln c'+1)$.

Fix some~$k>1$ and assume for the moment that all the first $k-1$ phases are typical. Then 
for $1\le i\le k-1$, we have 
$X_t\ge \sqrt{n}e^{-i/2}$ in Phase~$i$, \ie, when $en(i-1)\le t\le eni-1$. We analyze the event that additionally Phase~$k$ is 
typical, which subsumes the event $X_t\ge \sqrt{n}e^{-k/2}$ throughout Phase~$k$.
According to Lemma~\ref{lem:mgf-onemax}, we get in Phase~$i$, where $1\le i\le k$, 
\[
\expec{e^{\lambda \Delta_t} \mid X_t} \le 1+\lambda + 2\lambda e^{i/2}/\sqrt{n} + \littleo{\lambda/\ln n}
\le e^{\lambda +  \frac{2\lambda e^{i/2}}{\sqrt{n}} + \littleo{\frac{\lambda}{\ln n}}}
\]
The expression now depends on the time
only, therefore for $\lambda:=1/(en)$ 
\[
\prod_{t=0}^{enk-1} \expec{e^{\lambda \Delta_t} \mid X_0} \le 
e^{\lambda enk + \frac{2\lambda en}{\sqrt{n}}\sum_{i=1}^k e^{i/2} + enk \cdot \littleo{\frac{\lambda}{\ln n}}}
\le e^{k + \frac{6e^{k/2}}{n\sqrt{n}} + \littleo{1}} \le 
e^{k+o(1)},
\]
where we used that $k\le \ln n$. From Theorem~\ref{theo:main}.$(iv)$ for $a=\sqrt{n}e^{-k/2}$ and 
$t=enk-1$ we obtain
\[
\Prob(T_a < t) \le e^{k+o(1) - \lambda (g(X_0) - g(\sqrt{n}e^{-k/2}))}.
\]
From the proof of 
of Theorem~\ref{theo:expected-onemax-upper-lower}, the lower bound part, we already know 
that $g(X_0)\ge en\ln n -c''n$ for some constant~$c''>0$ (which is assumed large enough to subsume the $-O(\log n)$ term). 
Moreover, $g(x)\le en (\ln x+1)$ 
according to Lemma~\ref{lem:lowertailonemax-one}. We get
\[
\Prob(T_a < t) \le e^{k+o(1) - \ln n + O(1) - k /2 + (\ln n)/2} = e^{\frac{k-\ln n +O(1)}{2}} 
= c'''e^{k/2}/\sqrt{n},
\]
for some sufficiently large constant~$c'''>0$, 
which proves the bound on the probability of Phase~$k$ not being typical (without making statements about 
the earlier phases). 
The probability that all phases up to and including 
Phase~$k$ are typical is at least $1-(\sum_{i=1}^k c''' e^{i/2})/\!\sqrt{n} \ge 1-c' e^{k/2}/\!\sqrt{n}$
for a constant~$c'>0$. 
\end{proofof}

We now deduce a concentration inequality \wrt~linear functions,
essentially depending on all variables,
\ie, functions of the kind $f(x_1,\dots,x_n)=w_1 x_1 + \dots + w_n
x_n$, where $w_i\neq 0$. 
 This function class  contains \OneMax and
 has been studied intensely the last 15 years \cite{WittCPC13}.

\begin{theorem}
  The optimization time of the \ea on an arbitrary linear function
  with non-zero weights is at least $en\ln n - c n - ren$, where $c$
  is a constant, with probability at least $1-e^{-r/2}$ for any $r\ge
  0$. It is at most $en\ln n + (1+r) e n + O(1)$ with probability at
  least $1-e^{-r}$.
\end{theorem}

\begin{proof}
  The upper tail is proved in Theorem~5.1 in \cite{WittCPC13}. The
  lower bound follows from the lower tail in
  Theorem~\ref{theo:tails-onemax} in conjunction with the fact that
  the optimization time within the class of linear functions is
  stochastically smallest for \OneMax (Theorem~6.2
   in \cite{WittCPC13}).
\end{proof}

Finally, we consider $\LO(x_1,\dots,x_n):=\sum_{i=1}^n \prod_{j=1}^i
x_j$, another intensively studied standard benchmark problem
from the analysis of RSHs.  Tail bounds on the optimization time of
the \ea on \LO were derived in \cite{DJWZGECCO13}. This result
represents a fundamentally new contribution, but suffers from the fact
that it depends on a very specific structure and closed formula for
the optimization time.  Using a simplified version of
Theorem~\ref{theo:main} (see
Theorem~\ref{theo:main-simplifiedexponential}), it is possible to
prove similarly strong tail bounds without needing this exact formula.
As in \cite{DJWZGECCO13}, we are interested in a more general
statement. Let $T(a)$ be the number of steps until a \LO-value of
at least~$a$ is reached, where $0\le a\le n$.  Let $X_t:=\max\{0,a-\LO(x_t)\}$ be the distance
from the target~$a$ at time~$t$. Lemma~\ref{lem:drift-lo}
states the drift of $(X_t)_{t\in\N_0}$ 
exactly, see also \cite{DJWZGECCO13}. 

\begin{lemma}
\label{lem:drift-lo}
For all $i>0$, 
$\E{X_{t}-X_{t+1}\mid X_t=i} = (2-2^{-n+a-i+1})(1-1/n)^{a-i}(1/n)$.
\end{lemma}

\begin{proof}
The leftmost zero-bit is at position $a-i+1$. To increase the \LO-value (it cannot decrease), 
it is necessary 
to flip this bit and not to flip the first $a-i$ bits, which is reflected by 
the last two terms in the lemma. The first term is due to the expected number of 
free-rider bits (a sequence of previously random bits after the leftmost zero that happen to be all~$1$ 
at the time of improvement). Note that there can be between $0$ and $n-a+i-1$ such bits. By the 
usual argumentation using a geometric distribution, the expected number of 
free-riders in an improving step equals
\begin{equation*}
\sum_{k=0}^{n-a+i-1} k\cdot \left(\frac{1}{2}\right)^{\min\{n-a+i-1,k+1\}} 
= 1- 2^{-n+a-i+1},
\end{equation*}
hence the expected progress in an improving step is $2-2^{-n+a-i+1}$.
\end{proof}

We can now supply the tail bounds, formulated as Statements~$(ii)$ and 
$(iii)$ in the following theorem. The first statement is an exact expression 
for the expected optimization time, which has already been proved without 
drift analysis \cite{DJWZGECCO13}.

\begin{theorem}
Let $T(a)$ 
the time for the \ea to reach a \LO-value of at least~$a$. Moreover, let $r\ge 0$.  
Then
\begin{enumerate}[leftmargin=!,labelwidth=7mm]
\item[(i)]
$E(T(a)) = \frac{n^2-n}{2}\left(\left(1+\frac{1}{n-1}\right)^a-1\right)$.
\item[(ii)]
For 
$0<a\le n-\log n$, with probability at least
$1-e^{-\Omega(rn^{-3/2})}$ 
\[
T(a) \le \frac{n^2}{2}\left(\left(1+\frac{1}{n-1}\right)^a-1\right) + r.
\]
\item[(iii)]
For 
$\log^2 n-1\le a\le n$, with probability at least
$1-e^{-\Omega(rn^{-3/2})}-e^{-\Omega(\log^2 n)}$ 
\[
T(a) \ge \frac{n^2-n}{2}\left(\left(1+\frac{1}{n-1}\right)^{a}-1-\frac{2\log^2 n}{n}\right) - r.
\]
\end{enumerate}
\end{theorem}

\begin{proof}
The first statement is already contained in \cite{DJWZGECCO13} and proved without drift analysis.

We now turn to the second statement. From Lemma~\ref{lem:drift-lo}, 
$h(x) = (2-2/n)(1-1/n)^{a-x}/n$ is a lower bound 
on the drift $E(X_t-X_{t+1}\mid X_t=x)$ if $x\ge \log n$. To bound the change 
of the $g$-function, we observe that 
$h(x)\ge 1/(en)$ for all $x\ge 1$. This means 
that $X_{t}-X_{t+1} = k$ implies 
$g(X_t)-g(X_{t+1}) \le enk$. Moreover, to change 
the \LO-value by~$k$, it is necessary that 
\begin{itemize}
\item the first zero-bit flips (which has probability~$1/n$)
\item $k-1$ free-riders occur.
\end{itemize}
The change does only get stochastically larger if we assume an 
infinite supply of free-riders. Hence, $g(X_t)-g(X_{t+1})$ is 
stochastically dominated by a random variable $Z=en Y$, where 
$Y$
\begin{itemize}
\item is $0$ with probability $1-1/n$ and 
\item follows the geometric distribution with parameter~$1/2$ otherwise (where the support is $1,2,\dots$). 
\end{itemize}
The mgf. of $Y$ therefore equals 
\[
E(e^{\lambda Y}) = \left(1-\frac{1}{n}\right)e^0 + \frac{1}{n}\frac{1/2}{e^{-\lambda}-(1-1/2)} \le 1+\frac{1}{n(1-2\lambda)},
\]
where we have used $e^{-\lambda}\ge 1-\lambda$. For the  mgf. of $Z$ it follows
\[
E(e^{\lambda Z}) = E(e^{\lambda en Y}) \le 
1 + \frac{1}{n(1-2en\lambda)},
\]
hence
for $\lambda:=1/(4en)$ we get 
$D:=E(e^{\lambda Z})=1+2/n = 1+ 8e\lambda$, which means 
$D-1-\lambda = (8e-1)\lambda$. We get \[
\eta:=
\frac{\delta\lambda^2}{D-1-\lambda} = \frac{\delta  \lambda}{8e-1 } = \frac{\delta }{4en(8e-1)}
\]
(which is less than~$\lambda$ if $\delta\le 8e-1$) . Choosing $\delta:=n^{-1/2}$, we 
obtain 
$\eta=Cn^{-3/2}$ for $C:=1/((8e-1)(4e))$.
 
We set $t:=(\int_{\xmin}^{X_0} 1/h(x)\,\mathrm{d}x +r)/(1-\delta)$ 
in the first statement of  
Theorem~\ref{theo:main-simplifiedexponential}. The integral within~$t$ can 
be bounded according to 
\begin{align*}
U & :=\int_{\xmin}^{X_0} \frac{1}{h(x)}\,\mathrm{d}x 
\le \sum_{i=1}^a \frac{1}{(2-2/n)(1-1/n)^{a-i}/n} \\
& = \left(\frac{1}{2}+\frac{1}{2n-2}\right)\cdot n \cdot \frac{(1+1/(n-1))^a-1}{1/(n-1)}
= 
\frac{n^2}{2} \left(\left(1+\frac{1}{n-1}\right)^a - 1\right)
\end{align*}
Hence, using the theorem we get 
\[
\Prob(T>t) = \Prob(T > (U + r)/(1-\delta)) \le e^{-\eta r} \le e^{- Cr n^{-3/2}}.
\]
Since $U \le en^2$ and $1/(1-\delta)\le 1+2\delta = 1 + 2n^{-1/2}$, we get 
\[
\Prob(T \ge 
U + 2en^{3/2} + 2r) \le e^{-Cr n^{-3/2}}.
\]
Using the upper bound on~$U$ derived above, we obtain 
\[
\mathord{\Prob}\mathord{\left(T \ge 
\frac{n^2}{2} \left(\left(1+\frac{1}{n-1}\right)^a - 1\right) +r \right)}
 \le e^{-\Omega(r n^{-3/2})}
\]
as suggested.

Finally, we prove the third statement of this theorem in a quite symmetrical 
way to the second one. We can choose 
$h(x):=2(1-1/n)^{a-x}/n$ as an upper bound 
on the drift $E(X_t-X_{t+1}\mid X_t=x)$. The estimation of the $E(e^{\lambda Z})$ still applies. 
We set $t:= (\int_{\xmin}^{X_0} 1/h(x)\,\mathrm{d}x - r)/(1-\delta)$. Moreover, we assume 
$X_0\ge n-\log^2 n-1$, which happens with probability at least 
$1-e^{-\Omega(\log^2 n)}$. Note that 
\begin{align*}
L & := \int_{\xmin}^{X_0} \frac{1}{h(x)}\,\mathrm{d}x 
\ge \sum_{i=1}^{a-\log^2 n} \frac{1}{2(1-1/n)^{a-i}/n} \\
& = \frac{n^2-n}{2} \left(\left(1+\frac{1}{n-1}\right)^{a} - \left(1+\frac{1}{n-1}\right)^{\log^2 n}\right)\\
& \ge \frac{n^2-n}{2} \left(\left(1+\frac{1}{n-1}\right)^{a} - 1 -\frac{\log^2 n}{n}\right),
\end{align*}
where the last inequality used $e^{x}\le 1+2x$ for $x\le 1$ and $e^x\ge 1+x$ for $x\in \R$. 
The second statement of 
Theorem~\ref{theo:main-simplifiedexponential} yields (since state~$0$ is absorbing) 
\[
\Prob(T< t) = \Prob(T < (L - r)/(1+\delta)) \le e^{-\eta r} \le e^{- Cr n^{-3/2}}.
\]
Now, since 
\[
\frac{L-r}{1+\delta} \ge (L-r) - \delta(L-r) \ge L-r-en^{3/2}, 
\]
(using $L\le en^2$), 
we get the third statement by analogous calculations as above.
\end{proof}

\subsection{An Application to Probabilistic Recurrence Relations}
\label{sec:recurrences}
Drift analysis is not only useful in the theory of RSHs, but also in classical computer science.
Here, we study the
probabilistic recurrence relation $T(n)=a(n)+T(h(n))$, where $n$ is
the problem size, $a(n)$ the amount of work at the current level of
recursion, and $h(n)$ is a random variable, denoting the size of the
problem at the next recursion level. The asymptotic distribution 
(letting $n\rightarrow\infty$) of the number of cycles is well
studied \cite{Arratia1992}, but there are few results for finite $n$. Karp 
\cite{KarpJACM94} studied
this scenario using different probabilistic techniques than
ours. Assuming knowledge of $E(h(n))$, he proved upper tail bounds for
$T(n)$, more precisely he analyzed the probability of $T(n)$ exceeding
the solution of the ``deterministic'' process $T(n)=a(n)+T(E(h(n)))$.

We pick up the example from \cite[Section 2.4]{KarpJACM94} on the
number of cycles in a permutation~$\pi\in S_n$ drawn uniformly at
random, where $S_n$ denotes the set of all permutations of the $n$
elements $\{1,\dots,n\}$. A cycle is a
subsequence of indices $i_1,\dots,i_\ell$ such that
$\pi(i_j)=i_{(j\bmod \ell) +1}$ for $1\le j\le \ell$. Each permutation
partitions the elements into disjoint cycles. The expected number of
cycles in a random permutation is $H_n=\ln n +
\Theta(1)$. Moreover, it is easy to see that the length of the cycle
containing any fixed element is uniform on $\{1,\dots,n\}$. This gives
rise to the probabilistic recurrence $T(n)=1+T(h(n))$ expressing the random number of
cyles, where $h(n)$ is
uniform on $\{0,\dots,n-1\}$ .  As a result, \cite{KarpJACM94} shows
that the number of cycles is larger than $\log_2(n+1)+a$ with
probability at most $2^{-a+1}$. Note that the $\log_2(n)$, which
results from the solution of the deterministic recurrence, is already
by a constant factor away from the expected value. Lower tail
bounds are not obtained in \cite{KarpJACM94}.
Using our drift theorem (Theorem~\ref{theo:main}), 
it however follows that the number of cycles is sharply concentrated around 
its expectation.

\begin{theorem}
Let $N$ be the number of cycles in a random permutation of~$n$ elements.
Then 
\[
\Prob(N < (1-\epsilon)(\ln n)) \le e^{-\frac{\epsilon^2}{4} (1-o(1)) \ln n }
\]
for any constant $0<\epsilon<1$.
And for any constant $\epsilon>0$,
\[
\Prob(N \ge (1+\epsilon)((\ln n)+1)) \le e^{-\frac{\min\{\epsilon,\epsilon^2\}}{6} \ln n }.
\]
\end{theorem}

\begin{proof}
We regard the probabilistic recurrence as a stochastic process, where $X_t$, $t\ge 0$, denotes the number of elements not yet included in a cycle; $X_0=n$. 
As argued in \cite{KarpJACM94}, if $X_t=i$ then $X_{t+1}$ is uniform on $\{0,...,i-1\}$. Note that 
$N$ equals the first hitting time for  $X_t=0$, which is denoted by 
$T_0$ in our notation. Obviously, $N$ is stochastically larger than $T_a$ for any $a>0$. 

We now prove the lower tail using Theorem~\ref{theo:main}.$(iv)$. 
We compute $E(X_{t+1}\mid X_t)=(X_t-1)/2$, which means
$
E(X_t - X_{t+1} \mid X_t) \ge \frac{X_t}{2} = \frac{\lvert X_t\rvert}{2}
$
since only integral $X_t$ can happen. Therefore we choose $h(x)=\lvert x\rvert /2$ 
in the theorem. Letting $\xmin=1$, we obtain the drift function 
$g(i)=2+\int_1^i 2/\lceil j\rceil\,\mathrm{d}j = 
\sum_{j=1}^i 2/\lceil j \rceil$ for $i\ge 1$ and $g(0)=0$. We remark that 
other choices of $h$, with $1/2$ replaced by different constants, would lead to the essentially same result.

For the drift theorem, we have to compute 
$g(i)-g(X_{t+1})$, given $X_t=i$, and to bound the mgf.\ \wrt\ this difference. We get 
\[
g(i)-g(X_{t+1})
\le 
\begin{cases}
2(\ln(i)-\ln(j)) & \text{ for $j=1,...,i-1$, each with prob. $1/i$, }\\
2(\ln(i)+1) & \text{ with prob.\ $1/i$}
\end{cases}
\]

Let $X_t=i$. For $\lambda>0$, we bound the mgf.\
\begin{align*}
E(e^{\lambda (g(i)-g(X_{t+1}))})  & 
\le \frac{1}{i}\cdot e^{2\lambda} e^{2\lambda \ln i} 
+ \frac{1}{i}\sum_{j=1}^{i-1} e^{2\lambda (\ln i - \ln j)}
 = \frac{1}{i} e^{\eta} i^\eta  
 + \frac{1}{i} i^\eta \sum_{j=1}^{i-1} j^{-\eta},
\end{align*}
where $\eta=2\lambda$. Now assume $\eta$ constant and $\eta<1$. Then
\begin{align*}
E(e^{\lambda (g(i)-g(X_{t+1}))} ) & 
\le i^{\eta-1} e^{\eta} + i^{\eta-1} \left(1+\int_{1}^{i-1} j^{-\eta} \,\mathrm{d}j\right)\\
& \le i^{\eta-1} e^{\eta} + i^{\eta-1} \left(1+\left(\frac{1}{1-\eta}\left((i-1)^{1-\eta}- 1\right)\right)\right)\\
& \le i^{\eta-1} (e^{\eta}+1) + \frac{1}{1-\eta} - i^{\eta-1} 
= i^{\eta-1}e^\eta + \frac{1}{1-\eta} \\
& = 1 + i^{\eta-1}e^\eta + \frac{\eta}{1-\eta}
\le e^{e^\eta i^{\eta-1} + \frac{\eta}{1-\eta}} =:\beta
\end{align*}
using $1+x\le e^x$. The factor $e^{e^\eta i^{\eta-1}}$ will turn out 
to be negligible (more precisely, $e^{O((\ln n)^{\eta-1})}$) for $i\ge \ln n$ in the following, which is 
why we set $a:=\ln n$ in  Theorem~\ref{theo:main}.$(iv)$.

From the theorem, we get   
$\Prob(T_a < t) \le \beta^t e^{-\lambda(g(X_0)-g(a))}$. We work with the lower bound 
$g(X_0)-g(a) = \sum_{j=a}^n 2/j \ge 2(\ln(n+1)-\ln(a+1))$, which yields 
\begin{align*}
\Prob(T_a < t) & < \beta^t e^{-\lambda (2(\ln (n+1)-\ln (a+1))) }
= \beta^t e^{-\eta \ln n + O(\ln\ln n)} \\
& = e^{O(t(\ln n)^{\eta-1}) + \frac{\eta}{1-\eta} t -\eta\ln n + O(\ln\ln n)}
= e^{o(t) + O(\ln\ln n) + \frac{\eta}{1-\eta} t -\eta\ln n}
\end{align*}

Now we concentrate on the difference $d(\epsilon) = \frac{\eta}{1-\eta} t -\eta\ln n$
that is crucial for the order of growth of the last exponent. 
We assume $t:=(1-\epsilon)\ln n$ for some constant $\epsilon>0$ and set $\eta := \epsilon/2$ (implying 
$\epsilon<2$); hence $\lambda = \epsilon/4$. We get 
\begin{align*}
d(\epsilon) & = \frac{\epsilon/2}{1-\epsilon/2} (1-\epsilon) (\ln n) - \frac{\epsilon}{2} (\ln n)  
= \frac{\epsilon}{2} (\ln n) \left(\frac{1-\epsilon}{1-\epsilon/2}-1\right) 
 \le -\frac{\epsilon^2}{4} (\ln n)
\end{align*}

Using the bound for $d(\epsilon)$ in the exponent
and noting that $\epsilon>0$ is constant, give
$\Pr(T_a < (1-\epsilon) \ln n) \le e^{-\frac{\epsilon^2}{4} (1-o(1))\ln n }$,
which also bounds $T_0$ the same way.

To prove the upper tail, we must set $a:=0$ in Theorem~\ref{theo:main}.$(iii)$. Using the lower 
bound on the difference of $g$-values derived above, we estimate 
for $X_t=i$ and any $\lambda>0$ 
\[
E(e^{-\lambda (g(i)-g(X_{t+1}))} ) \le 
\frac{1}{i}\sum_{j=0}^{i-1} e^{-\lambda (2(\ln (i+1)-\ln (j+1))) } 
= \frac{1}{i}\sum_{j=0}^{i-1} \left(\frac{j+1}{i+1}\right)^\eta,
\]
where again $\eta=2\lambda$. Hence, similarly to the estimations 
for the lower tail, 
\[
E(e^{-\lambda (g(i)-g(X_{t+1}))} ) \le  \frac{1}{i^{\eta+1}} \int_1^{i} j^\eta \,\mathrm{d}j
\le  \frac{1}{i^{\eta+1}} \frac{1}{\eta+1} i^{\eta+1}  = \frac{1}{\eta+1} 
\le e^{-\frac{\eta}{\eta+1}} =: \beta
\]
From the drift theorem, we get
\[
 \Prob(T_0> t) \le \beta^t e^{\lambda (g(X_0)-g(0))} 
\le e^{-\frac{\eta t}{\eta+1}} e^{\lambda (2(\ln(n)+1))} = 
e^{-\frac{\eta t}{\eta+1} + \eta (\ln n + 1) }.
\]
Setting $t:=(1+\epsilon)(\ln n+1)$ and 
$\eta=\epsilon/2$, the exponent is no more than
\[
-\frac{\eta (1+\epsilon/2+\epsilon/2)(\ln n+1)}{1+\epsilon/2} + \eta(\ln n + 1)
\le -\frac{\epsilon^2 }{4+2\epsilon} (\ln n +1).
\]
The last fraction is at most $-\frac{\epsilon^2}{6}$ if $\epsilon\le 1$ and 
at most $-\frac{\epsilon}{6}$ otherwise (if $\epsilon>1$). 
Altogether 
\[
\Prob(T_0> t\mid X_0=n) \le e^{-\frac{\min\{\epsilon^2,\epsilon\}}{6}(\ln n+1)}.
\]
\end{proof}
For appropriate functions~$g(x)$, our drift theorem may provide sharp 
concentration results for other probabilistic recurrences, such as the case
$a(n)>1$.

\section{Conclusions}
We have presented a new  and versatile
drift theorem with tail bounds. It can be understood as a general
variable drift theorem and can be specialized into all existing
variants of variable, additive and multiplicative drift theorems we
found in the literature as well as the fitness-level technique. 
 Moreover, it provides lower and upper tail
bounds, which were not available before in the context of variable
drift. Tail bounds were used to prove sharp concentration inequalities
on the optimization time of the \ea on \OneMax, linear functions and
\LO. Despite the highly random fashion this RSH operates, its
optimization time is highly concentrated up to lower order
terms. The drift theorem also leads to tail bounds on the number of
cycles in random permutations.  The proofs illustrate
how to use the tail bounds and we provide simplified (specialized)
versions of the corresponding statements. We believe that this research
helps consolidate the area of drift analysis. The
general formulation of drift analysis increases our understanding of
its power and limitations. The tail bounds imply more practically
useful statements on the optimization time than the expected
time. We expect further applications of our theorem, also to 
classical randomized algorithms.

\subsubsection*{Acknowledgements.} This
research received funding from the
European Union Seventh Framework Programme (FP7/2007-2013) under
grant agreement no 618091 (SAGE) and from the Danish Council 
for Independent Research (DFF-FNU) under grant no.\ 4002--00542.


\newpage

\appendix

\section{Existing Drift Theorems as Special Cases}
\label{sec:specialcasesvariable}
In this appendix, we show that virtually all known variants of drift
theorems with drift towards the target can be derived from our general
Theorem~\ref{theo:main} with surprisingly short proofs.

\subsection{Variable Drift and Fitness Levels}
A clean form of a variable drift theorem, generalizing 
previous formulations from \cite{Johannsen10} and \cite{MitavskiyVariable}, 
was presented in \cite{RoweSudholtChoice}. We restate 
this theorem in our notation and carry out two generalizations: we allow for a continuous, unbounded  
state space instead of demanding a finite one and do not fix $\xmin=1$. 

\begin{theorem}[Variable Drift, Upper Bound; following \cite{RoweSudholtChoice}] 
\label{theo:variable-rowe-sudholt}
Let $(X_t)_{t\in\N_0}$, be a stochastic process over some state space $S\subseteq \{0\}\cup \R_{\ge \xmin}$, where 
  $\xmin > 0$. 
Let  $h\colon \R_{\ge \xmin}\to\R^+$ be a monotone increasing function such that 
$1/h(x)$ is integrable on $\R_{\ge \xmin}$ and 
$E(X_t-X_{t+1} \mid \filt) \ge h(X_t)$ if $X_t\ge \xmin$.
 Then it holds for the first hitting time 
$T:=\min\{t\mid X_t=0\}$ that 
\[
E(T\mid \filtzero) \le 
\frac{\xmin}{h(\xmin)} + \int_{\xmin}^{X_0} \frac{1}{h(x)} \,\mathrm{d}x.
\] 
\end{theorem}

\begin{proof}
Since $h(x)$ is monotone increasing, $1/h(x)$ is decreasing 
and $g(x)$, defined in Remark~\ref{remark:function-h}, is concave. By Jensen's inequality, we 
get
\begin{align*}
& E(g(X_t)-g(X_{t+1})  \mid \filt) 
 \;\ge \;
g(X_t) - g(E(X_{t+1}\mid \filt)) \\
& \;=\;
\int_{E(X_{t+1}\mid \filt)}^{X_t} \frac{1}{h(x)} \,\mathrm{d}x 
\;\ge\; 
\int_{X_t-h(X_t)}^{X_t} \frac{1}{h(x)} \,\mathrm{d}x, 
\end{align*}
where the equality just expanded $g(x)$. 
Using that $1/h(x)$ is decreasing, it follows 
\[
\int^{X_t}_{X_t-h(X_t)} \frac{1}{h(x)} \,\mathrm{d}x 
\ge \int^{X_t}_{X_t-h(X_t)} \frac{1}{h(X_t)} \,\mathrm{d}y   = 
\frac{h(X_t)}{h(X_t)} = 1.
\] 
Plugging in $\alpha_{\mathrm{u}}:=1$ in Theorem~\ref{theo:main} completes the proof.
\end{proof}

 In \cite{RoweSudholtChoice} it was also pointed out that variable drift theorems
 in discrete search spaces look similar to bounds obtained from
 the fitness level technique (also called the method of $f$-based
 partitions, first formulated in \cite{WegenerICALP01}). For the
 sake of completeness, we present the classical upper bounds by fitness
 levels \wrt~the \ea here and prove them by drift analysis.

 \begin{theorem}[Classical Fitness Levels, following \cite{WegenerICALP01}]
 \label{theo:fitnesslevels}
 Consider the \ea maximizing some function~$f$ and a partition of the search space into non-empty sets $A_1,\dots,A_m$. Assume 
 that the sets form an $f$-based partition, \ie, 
 for $1\le i<j\le m$  and all $x\in A_i$, $y\in A_j$ it holds $f(x)<f(y)$. Let $p_i$ be a lower bound 
 on the probability that a search point in~$A_i$ is mutated into a
 search point in~$\cup_{j=i+1}^m A_{j}$. 
 Then the expected hitting time of~$A_m$ is at most
 $\sum_{i=1}^{m-1} \frac{1}{p_i}.$
 \end{theorem}

 \begin{proof}
 At each point of time, the \ea is in a unique fitness level. Let $Y_t$ the current fitness level at time~$t$. 
 We consider the process defined by $X_t=m-Y_t$. 
 By definition of fitness levels and the \ea, $X_t$ is non-increasing over time. Consider $X_t=k$ for $1\le k\le m-1$. 
 With 
 probability $p_{m-k}$, the $X$\nobreakdash-value decreases by at least~$1$. 
 Consequently, 
 $\E{X_t-X_{t+1}\mid X_t = k}  \ge p_{m-k}$. We define $h(x)=p_{m-\lceil x\rceil}$, $\xmin=1$ and 
 $\xmax=m-1$  
 and obtain an integrable, monotone increasing function on $[\xmin,\xmax]$. Hence, the upper bound 
 on $E(T\mid \filtzero)$ 
 from  
 Theorem~\ref{theo:variable-rowe-sudholt} becomes at most $\frac{1}{p_{1}} + \sum_{i=1}^{m-2} \frac{1}{p_{m-i}}$, 
 which 
 completes the proof.
 \end{proof}

Recently, the fitness-level technique was considerably refined and supplemented 
by lower bounds \cite{SudholtTEC13}. We can also identify these extensions as a special 
case of general drift. This material is included in  
an subsection on its own, 
\ref{app:fitness}. Similarly, we have moved a treatment of variable 
drift theorems with non-monotone drift to \ref{app:variablenonmonotone}.

Finally, so far only two theorems dealing with upper bounds on variable drift and 
thus lower bounds on the hitting time seems to have been published. The first one 
was derived 
in \cite{DFWVariable}. Again, we present a variant without unnecessary assumptions, 
more precisely 
we allow continuous state spaces and use less restricted $c(x)$ and $h(x)$. 

\begin{theorem}[Variable Drift, Lower Bound; following \cite{DFWVariable}] 
\label{theo:variable-dfw}
Let $(X_t)_{t\in\N_0}$, be a stochastic process over some state space $S\subseteq \{0\}\cup [\xmin,\xmax]$, where 
  $\xmin > 0$. 
Suppose there exists two functions $c,h\colon [\xmin,\xmax]\to\R^+$ such that 
 $h(x)$ is monotone increasing and $1/h(x)$ integrable on $[\xmin,\xmax]$, and 
for all $t\ge 0$, 
\begin{enumerate}[leftmargin=!,labelwidth=7mm]
\item[(i)] $X_{t+1} \le X_t$,
\item[(ii)] 
$X_{t+1} \ge c(X_t)$ for $X_t\ge \xmin$,
\item[(iii)] 
$E(X_t-X_{t+1} \mid \filt) \le h(c(X_t))$ for $X_t\ge \xmin$.
\end{enumerate}
 Then it holds for the first hitting time 
$T:=\min\{t\mid X_t=0\}$ that 
\[
E(T\mid \filtzero) \ge 
\frac{\xmin}{h(\xmin)} + \int_{\xmin}^{X_0} \frac{1}{h(x)} \,\mathrm{d}x.
\] 
\end{theorem}

\begin{proof}
Using the definition of~$g$ according to Remark~\ref{remark:function-h}, we compute the drift
\begin{align*}
& E(g(X_t)-g(X_{t+1})  \mid \filt) 
 \;=\; 
\mathord{E}\mathord{\left(\int_{X_{t+1}}^{X_{t}} \frac{1}{h(x)} \,\mathrm{d}x \mid \filt \right)} \\
& \;\le\; \mathord{E}\mathord{\left(\int_{X_{t+1}}^{X_{t}} \frac{1}{h(c(X_t))} \,\mathrm{d}x \mid \filt\right)},
\end{align*}
where we have used that $X_t\ge X_{t+1}\ge c(X_t)$ and that 
$h(x)$ is monotone increasing. The last integral equals 
\[
\frac{X_t-E(X_{t+1}\mid \filt)}{h(c(X_t))}  \le \frac{h(c(x))}{h(c(x))}
\;=\;1.
\]
Plugging in $\alpha_{\mathrm{\ell}}:=1$ in 
Theorem~\ref{theo:main} completes the proof.
\end{proof}

Very recently, Theorem~\ref{theo:variable-dfw} was relaxed in \cite{GiessenWittAlgo18}  by 
replacing the deterministic condition $X_{t+1} \ge c(X_t)$ by a probabilistic 
one. We note without proof that also this generalization can be 
proved with Theorem~\ref{theo:main}.

\subsection{Multiplicative Drift}
\label{sec:specialcasesother}
We continue by showing that Theorem~\ref{theo:main} can be specialized in order to re-obtain 
other classical and recent variants of drift theorems. Of course, Theorem~\ref{theo:main} 
is a generalization of additive drift (Theorem~\ref{theo:additive}), which interestingly 
was used to prove the general theorem itself. The remaining important (in fact possibly 
the most important) strand of drift theorems 
is therefore represented by so-called multiplicative drift, which we focus on in this subsection.  
Roughly speaking, the underlying assumption is that the progress to the optimum is proportional to the distance (or can be bounded 
in this way). 
Early theorems covering this scenario, without using the notion of 
of drift, can be found in \cite{BaritompaS96}.

The following theorem 
is the strongest variant of the multiplicative drift 
theorem (originally introduced by \cite{DJWMultiplicativeAlgorithmica}), which can be found in \cite{DoerrGoldbergAdaptive}.   
It was used to analyze RSHs on combinatorial optimization problems and linear functions. 
Here we also need a tail bound from our main theorem (more precisely, 
the third item in Theorem~\ref{theo:main}). Note that the multiplicative 
drift theorem requires $\xmin$ to be positive, \ie, a gap in the state space. Without the 
gap, no finite first hitting time can be proved from the prerequisites of 
multiplicative drift.

\begin{theorem}[Multiplicative Drift, Upper Bound; following \cite{DoerrGoldbergAdaptive}]
\label{theo:multiplicative-drift}
Let $(X_t)_{t\in\N_0}$, be a stochastic process over some state space $S\subseteq \{0\}\cup \R_{\ge \xmin}$, where 
  $\xmin > 0$. 
 Suppose that there exists some $\delta$, where $0<\delta<1$ such that 
$E(X_t-X_{t+1} \mid \filt) \ge \delta X_t$. Then the following statements 
hold for the first hitting time $T:=\min\{t\mid X_t=0\}$.
\begin{enumerate}[leftmargin=!,labelwidth=7mm]
\item[(i)]  $E(T\mid \filtzero) \le \frac{\ln(X_0/\xmin)+1}{\delta}$.
\item[(ii)] $\Prob(T \ge \frac{\ln(X_0/\xmin) + r}{\delta} \mid \filtzero)\le e^{-r}$ for all $r> 0$.
\end{enumerate}
\end{theorem}

\begin{proof}
  Choosing $h(x)=\delta x$, the process satisfies Condition~$(i)$ of
  Corollary~\ref{cor:tailbound-h-convex-concave}, which implies that 
  \begin{align*}
    \expec{T\mid \filtzero}\leq \frac{\xmin}{\delta\xmin} +
    \int_{\xmin}^{X_0}\frac{1}{\delta y}\,\mathrm{d}y = \frac{\ln(X_0/\xmin)+1}{\delta},
  \end{align*}
	which proves the first item from the theorem.
  The process also satisfies Condition~$(iii)$ of
  Corollary~\ref{cor:tailbound-h-convex-concave}, however, this will result in 
	the loss of a factor~$e$ in the tail bound. We argue directly instead.

 Using the notation from Theorem~\ref{theo:main}, we choose $h(x) = \delta x$ and obtain 
 $E(X_t-X_{t+1} \mid \filt) \ge h(X_t)$ by the prerequisite on 
 multiplicative drift. Moreover, according to Remark~\ref{remark:function-h} we define 
 $g(x) = \xmin/(\delta \xmin) + \int_{\xmin}^x 1/(\delta y) \,\mathrm{d}y = 
 1/\delta + \ln(x/\xmin) / \delta$ for $x\ge \xmin$. We set $a:=0$ and consider 
 \begin{align*}
 E(e^{-\delta(g(X_t)-g(X_{t+1}))}\mid \filt; X_t\ge \xmin)  
 & = E(e^{ \ln(X_{t+1}/\xmin) - \ln(X_t/\xmin)})\mid \filt; X_t\ge \xmin) \\
 & = E((X_{t+1}/X_t) \mid \filt; X_t\ge \xmin) \le 1-\delta,
 \end{align*}
 Hence, we can choose $\beta_{\mathrm{u}}(t)=1-\delta$ for all $X_t\ge \xmin$ and $\lambda=\delta$ in the third item of Theorem~\ref{theo:main} 
 to obtain 
 \[
 \Prob(T>t \mid \filtzero) < (1-\delta)^{t} \cdot e^{\delta (g(X_0)-g(\xmin))} \le e^{-\delta t+\ln(X_0/\xmin)}.
 \]
 Now the second item of Theorem~\ref{theo:multiplicative-drift} follows by choosing $t:=(\ln(X_0/\xmin) + r)/\delta$.	
\end{proof}

Compared to the upper bound, the following 
lower-bound includes a
condition on the maximum step-wise progress and requires 
non-increasing sequences. It generalizes the 
version in \cite{WittCPC13} and its predecessor in 
\cite{LehreWittAlgorithmica12} 
by not assuming
$\xmin\ge 1$.

\begin{theorem}[Multiplicative Drift, Lower Bound; following \cite{WittCPC13}]
\label{theo:multdrift-lower}
Let $(X_t)_{t\in\N_0}$, be a stochastic process over some state space $S\subseteq \{0\}\cup [\xmin,\xmax]$, where 
  $\xmin > 0$.  Suppose that there exist $\beta,\delta$, where $0< \beta,\delta\le 1$ such that
for all $t\ge 0$
\begin{enumerate}[leftmargin=!,labelwidth=7mm]
\item[(i)]
$X_{t+1} \le X_t$,
\item[(ii)]
$\Prob(X_t-X_{t+1}\ge \beta  X_t) \;\le\; \frac{\beta\delta}{1+\ln (X_t/\xmin)}$.
\item[(iii)]
$E(X_t-X_{t+1} \mid \filt) \le \delta X_t$.
\end{enumerate}
 Define the first hitting time 
  $T:=\min\{t\mid X_t=0\}$. Then
\[
  \E{T\mid \filtzero}\;\ge\; \frac{\ln(X_0/\xmin)+1}{\delta}\cdot \frac{1-\beta}{1+\beta}.\]
\end{theorem}

\begin{proof}
Using the definition of~$g$ according to Remark~\ref{remark:function-h}, we compute the drift
\begin{align*}
& E(g(X_t)-g(X_{t+1})  \mid \filt) 
 \;=\; 
\mathord{E}\mathord{\left(\int_{X_{t+1}}^{X_{t}} \frac{1}{h(x)} \,\mathrm{d}x \mid \filt \right)} \\
& \;\le\; 
\mathord{E}\mathord{\left(\int_{X_{t+1}}^{X_{t}} \frac{1}{h(x)} \,\mathrm{d}x \filtcond{ X_{t+1}\ge (1-\beta)X_t} \right)} 
\cdot \Prob(X_{t+1}\ge (1-\beta) X_t) \\
& \qquad\qquad + g(X_t) \cdot (1- \Prob(X_{t+1}\ge (1-\beta)X_t))
\end{align*}
where we used the law of total probability and $g(X_{t+1}) \ge 0$. 
As in the proof of Theorem~\ref{theo:multiplicative-drift}, we have 
$g(x)=(1 + \ln(x/\xmin))/\delta$. 
Plugging in $h(x)=\delta x$, using the bound on $\Prob(X_{t+1}\ge (1-\beta)X_t)$ 
and $X_{t+1}\le X_t$, the drift is further bounded by 
\begin{align*}
& \mathord{E}\mathord{\left(\int_{X_{t+1}}^{X_{t}} \frac{1}{\delta (1-\beta)X_t} \,\mathrm{d}x \mid \filt\right)}
+ \frac{\beta\delta}{1+\ln (X_t/\xmin)} \cdot \frac{1+\ln (X_t/\xmin)}{\delta} \\
& \;=\; \frac{E(X_t-X_{t+1}\mid \filt)}{\delta (1-\beta)X_t} + \beta \;\le\; \frac{\delta X_t}{\delta (1-\beta)X_t} + \beta 
\;\le\; \frac{1+\beta}{1-\beta},
\end{align*}
Using $\alpha_\ell=(1+\beta)/(1-\beta)$ and expanding $g(X_0)$,
the proof is complete.
\end{proof}

Very recently, 
it was shown in \cite{DoerrDoerrKoetzingGECCO16} 
that the 
 monotonicity condition $X_{t+1} \le X_t$ 
in Theorem~\ref{theo:multdrift-lower} 
can be dropped if item~$(iii)$ is replaced by 
$E(s-X_{t+1} \cdot \indic{X_{t+1}\le s} \mid \filt) \le \delta s$ 
for all $s\le X_t$. We note without proof that also this 
strengthened theorem can be obtained from our general drift theorem.

\section{Fitness Levels Lower and Upper Bounds as Special Cases}
\label{app:fitness}
We pick up the consideration of fitness levels again and 
prove the following lower-bound theorem due to Sudholt 
\cite{SudholtTEC13} by drift analysis. 
See Sudholt's paper for possibly undefined or unknown terms.

 \begin{theorem}[Theorem~3 in \cite{SudholtTEC13}]
   \label{theo:fitnesslevells-lower}
 Consider an algorithm~$\mathcal{A}$ and a partition of the search space into non-empty sets $A_1,\dots,A_m$. For a mutation-based 
 EA $\mathcal{A}$ we again say that $\mathcal{A}$ is in $A_i$ or on level~$i$ if the best individual created so far 
 is in $A_i$. Let the probability of $\mathcal{A}$ traversing from level~$i$ to level~$j$ in one step be at most
 $u_i\cdot \gamma_{i,j}$ and $\sum_{j=i+1}^m \gamma_{i,j}=1$. Assume that for all $j>i$ and some $0\le \chi\le 1$ it holds
 \begin{equation}
 \label{eq:cond-chi-fitness}
 \gamma_{i,j} \ge \chi \sum_{k=j}^m \gamma_{i,k}.
 \end{equation}
 Then the expected hitting time of~$A_m$ is at least
 \begin{align*}
 & \sum_{i=1}^{m-1} \Prob(\text{$\mathcal{A}$ starts in $A_i$}) \cdot \left(\frac{1}{u_i}+\chi \sum_{j=i+1}^{m-1} \frac{1}{u_j}\right) \\
 & \ge  \sum_{i=1}^{m-1} \Prob(\text{$\mathcal{A}$ starts in $A_i$}) \cdot \chi \sum_{j=i}^{m-1} \frac{1}{u_j}.
 \end{align*}
 \end{theorem}

 \begin{proof}
 Since $\chi\le 1$, the second lower bound follows immediately from the first one, which we prove in the following. To adopt the perspective 
 of minimization, we say that $\mathcal{A}$ is on distance level~$m-i$ if the best individual created so far 
 is in~$A_i$. Let $X_t$ be the algorithm's distance level at time~$t$. We define the drift function $g$ mapping 
 distance levels to non-negative numbers (which then form a new stochastic process) by 
 \[
 g(m-i) =   \frac{1}{u_i} + \chi\sum_{j=i+1}^{m-1} \frac{1}{u_j}
 \]
 for $1\le i\le m-1$. Defining $u_m:=\infty$, we extend the function to $g(0)=0$. Our aim is to prove that the drift 
 \[
 \Delta_{t}(m-i) := E(g(m-i)-g(X_{t+1}) \mid X_t=m-i)
 \]
 has expected value at most~$1$. Then the theorem follows immediately using additive drift (Theorem~\ref{theo:additive}) 
 along with the law of total probability 
 to condition on the starting level.

 To analyze the drift, consider the case that the distance level decreases from $m-i$ to $m-\ell$, where $\ell > i$. We obtain
 \[
 g(m-i) - g(m-\ell) = \frac{1}{u_i}-\frac{1}{u_{\ell}} + \chi \sum_{j=i+1}^{\ell}  \frac{1}{u_j},
 \]
 which by the law of total probability (and as the distance level cannot increase) implies
 \begin{align*}
 \Delta_{t}(m-i) & = \sum_{\ell=i+1}^{m} u_i\cdot \gamma_{i,\ell} 
 \left(\frac{1}{u_i}-\frac{1}{u_{\ell}} + \chi \sum_{j=i+1}^{\ell}  \frac{1}{u_j}\right)\\
 & = 1 + 
  u_i \sum_{\ell=i+1}^{m} \gamma_{i,\ell} 
 \left(-\frac{1}{u_{\ell}} + \chi \sum_{j=i+1}^{\ell}  \frac{1}{u_j}\right),
 \end{align*}
 where the last equality used 
 $\sum_{\ell=i+1}^m \gamma_{i,\ell}=1$. If we can prove that 
 \begin{equation}
 \label{eq:fitness-rearr}
 \sum_{\ell=i+1}^{m}  \gamma_{i,\ell} 
 \chi \sum_{j=i+1}^{\ell}  \frac{1}{u_j}
 \le \sum_{\ell=i+1}^{m}  \gamma_{i,\ell}  \cdot 
 \frac{1}{u_{\ell}}
 \end{equation}
 then $\Delta_{t}(m-i) \le 1$ follows and the proof is complete. To show this, observe 
 that 
 \[
 \sum_{\ell=i+1}^{m}  \gamma_{i,\ell} 
 \chi \sum_{j=i+1}^{\ell}  \frac{1}{u_j}
 = \sum_{j=i+1}^m \frac{1}{u_j}\cdot \chi\sum_{\ell=j}^{m}\gamma_{i,\ell} 
 \]
 since the term $\tfrac{1}{u_j}$ appears for all terms $\ell=j,\dots,m$ in the outer sum, 
 each term weighted by $\gamma_{i,\ell}\chi$. By \eqref{eq:cond-chi-fitness}, 
 we have $\chi\sum_{\ell=j}^{m}\gamma_{i,\ell} \le \gamma_{i,j}$, and 
 \eqref{eq:fitness-rearr} follows.
 \end{proof}

We remark here without going into the details that also the refined 
upper bound by fitness levels (Theorem~4 in \cite{SudholtTEC13})
can be proved using general drift.

\section{Non-monotone Variable Drift}
\label{app:variablenonmonotone}

In many applications, a monotone increasing function $h(x)$ 
bounds the drift from below. For example, the expected 
progress towards the optimum of \OneMax increases with 
the distance of the current search point from the optimum. 
However, certain ant colony optimization algorithms do not have this 
property and exhibit 
a non-monotone drift \cite{DoerrHotaKoetzingGECCO12}. To handle this case,  generalizations of the variable drift theorem 
have been developed that 
does not require $h(x)$ to be monotone. The most recent 
version of this theorem is presented in \cite{FeldmannKoetzingFOGA13}. 
Unfortunately, it turned out that the two generalizations suffer from 
a missing condition, relating positive and negative drift to each other. 
Adding the condition and 
removing an unnecessary assumption (more precisely, the continuity of $h(x)$) 
the theorem in \cite{FeldmannKoetzingFOGA13} can be  corrected 
as follows. Note that this formulation is also used in \cite{OlivetoWittTCS15} 
but proved with a specific drift theorem instead of our general approach.

\begin{theorem}[extending \cite{FeldmannKoetzingFOGA13}]
\label{theo:variablenonmonotone}
Let $(X_t)_{t\in\N_0}$, be a stochastic process over some state space $S\subseteq \{0\}\cup \R_{\ge \xmin}$, where 
  $\xmin > 0$. 
Suppose there exists two functions  $h,d\colon  \R_{\ge \xmin} \to \R^+$, where 
$1/h$ is integrable, and 
a constant $c\ge 1$  such that for all $t\ge 0$
\begin{enumerate}
\renewcommand{\labelenumi}{(\arabic{enumi})}
\item 
\label{it:i}
$E(X_t-X_{t+1} \filtcond{ X_t\ge \xmin}) \ge h(X_t)$, 
\item
\label{it:ii}
$\frac{E((X_{t+1} - X_{t}) \cdot  \indic{X_{t+1}> X_t} \filtcond{
    X_t\ge \xmin})}{E((X_t - X_{t+1}) \cdot  \indic{X_{t+1}<X_t} \filtcond{X_t\ge \xmin})} \le \frac{1}{2c^2}$,
\item
\label{it:iii}
$\lvert X_t-X_{t+1}\rvert\le d(X_t)$ if $X_t\ge \xmin$,
\item 
\label{it:iv}
 for all $x,y\ge \xmin$ with $\lvert x-y\rvert \le d(x)$, it holds $h(\min\{x,y\}) \le c h(\max\{x,y\})$.
\end{enumerate}
 Then it holds for the first hitting time 
$T:=\min\{t\mid X_t=0\}$ that 
\[
E(T\mid \filtzero) \le 
2c\left(\frac{\xmin}{h(\xmin)} + \int_{\xmin}^{X_0} \frac{1}{h(x)} \,\mathrm{d}x\right).
\] 
\end{theorem}

It is worth noting that Theorem~\ref{theo:variable-rowe-sudholt} is not necessarily 
a special case of Theorem~\ref{theo:variablenonmonotone}.

\begin{proofof}{Proof of Theorem~\ref{theo:variablenonmonotone}}
Using the definition of~$g$ according to Remark~\ref{remark:function-h} and assuming $X_t\ge \xmin$, we compute the drift
\begin{align*}
& E(g(X_t)-g(X_{t+1})  \mid \filt) 
 \;=\; 
\mathord{E}\mathord{\left(\int_{X_{t+1}}^{X_{t}} \frac{1}{h(x)} \,\mathrm{d}x\bigm| \filt\right)} \\
& \;=\; \mathord{E}\mathord{\left(\int_{X_{t+1}}^{X_{t}} \frac{1}{h(x)} \,\mathrm{d}x \cdot \indic{X_{t+1}<X_t}\bigm| \filt\right)} 
- \mathord{E}\mathord{\left(\int_{X_{t}}^{X_{t+1}} \frac{1}{h(x)} \,\mathrm{d}x\cdot\indic{X_{t+1} > X_t}\bigm| \filt\right)},
\end{align*}
where equality holds since the integral is empty if $X_{t+1}=X_t$. 
Item~\eqref{it:iv} from the prerequisites 
yields $h(x)\le c h(X_t)$ if $X_t-d(X_t)\le x<X_t$ and $h(x)\ge h(X_t)/c$ if $X_t<x\le X_t+d(X_t)$. 
Using this and $\lvert X_{t}-X_{t+1}\rvert  \le d(X_t)$, the drift can be further bounded by
\begin{align*}
& \mathord{E}\mathord{\left(\int_{X_{t+1}}^{X_{t}} \frac{1}{ch(X_t)} \,\mathrm{d}x \cdot \indic{X_{t+1}<X_t}\bigm| \filt\right)} 
- 
\mathord{E}\mathord{\left(\int_{X_t}^{X_{t+1}} \frac{c}{h(X_t)} \, \mathrm{d}x \cdot \indic{X_{t+1}>X_t}\bigm| \filt\right)} \\
& \;\ge\; 
\mathord{E}\mathord{\left(\int_{X_t}^{X_{t+1}} \frac{1}{2ch(X_t)} \, \mathrm{d}x \cdot \indic{X_{t+1}<X_t}\bigm| \filt\right)} 
\; = \; 
\frac{E((X_t-X_{t+1}\bigm| \filt)\cdot \indic{X_{t+1}<X_t}) }{2c h(X_t)}\\
& \;\ge\;  \frac{h(X_t)}{2c h(X_t)} \;=\;
\frac{1}{2c}, 
\end{align*} 
where the first inequality used the Item~\eqref{it:ii} from the prerequisites and the last one 
Item~\eqref{it:i}. 
Plugging in $\alpha_{\mathrm{u}}:=1/(2c)$ in 
Theorem~\ref{theo:main} completes the proof.
\end{proofof}

\section{Standard Results}
\begin{theorem}[Dominated Convergence Theorem,
  \cite{Williams1991ProbabilityMartingales} (page 88)]
	\label{thm:dot}
  Given a stochastic process $(S_t)_{t\in\mathbb{N}}$, two random
  variables $X$ and $V$, and a sub-$\sigma$-algebra $\mathcal{G}$ of
  $\mathcal{F}$. If
  \begin{enumerate}
  \item[1)] $|S_n(\omega)|\leq V(\omega)$ for all $\omega\in\Omega$,
  \item[2)] $\expect{V}<\infty$, and
  \item[3)] $\lim_{t\rightarrow\infty} S_t\rightarrow S $ almost surely,
  \end{enumerate}
  then $\lim_{t\rightarrow\infty}\expect{S_t\mid  \mathcal{G}}\rightarrow \expect{S\mid\mathcal{G}}$.  
\end{theorem}

\end{document}